\documentclass[10pt,twocolumn,letterpaper]{article}

\usepackage{cvpr}              

\usepackage{amsmath,amsfonts,bm}

\def\eqref#1{equation~\ref{#1}}

\def\1{\bm{1}}

\def\rd{{\textnormal{d}}}

\def\rmI{{\mathbf{I}}}

\DeclareMathAlphabet{\mathsfit}{\encodingdefault}{\sfdefault}{m}{sl}
\SetMathAlphabet{\mathsfit}{bold}{\encodingdefault}{\sfdefault}{bx}{n}

\newcommand{\Nc}{{\mathcal N}}

\def\rmI{{\mathbf{I}}}

\newcommand{\innerprod}[2]{\left<{#1},{#2}\right>}

\newcommand{\x}{{\boldsymbol x}}

\newcommand{\cb}{{\boldsymbol c}}

\newcommand{\epsilonb}{{\boldsymbol \epsilon}}

\definecolor{cfg}{rgb}{0.906, 0.435, 0.318}
\definecolor{cfgpp}{rgb}{0.165, 0.616, 0.561}
\definecolor{cfgnull}{rgb}{0.208, 0.565, 0.953}

\newcommand{\epscond}{{\hat\epsilonb_\cb}}

\newcommand{\epsnullnc}{{\hat\epsilonb_\varnothing}}
\newcommand{\epscfgnc}{{\hat\epsilonb_\cb^\omega}}

\newcommand{\tweediecfgnc}{{\hat\x_\cb^\omega}}

\usepackage{stfloats}
\usepackage{amsthm}
\usepackage{amsmath}
\usepackage{caption}
\usepackage{wrapfig} 
\usepackage{makecell}

\newtheorem{theorem}{Theorem}
\newtheorem{lemma}{Lemma}
\newtheorem{definition}{Definition}
\usepackage[T1]{fontenc}
\usepackage{url}
\usepackage{algpseudocode}
\usepackage{algorithm}
\usepackage{mathrsfs}
\usepackage{amssymb}
\usepackage{graphicx}
\usepackage{subcaption}
\usepackage{float}
\usepackage{booktabs}
\usepackage{kotex}
\usepackage{bm}
\usepackage{multirow}
\usepackage{microtype}      
\usepackage{xcolor}   

\usepackage[utf8]{inputenc} 
\usepackage[T1]{fontenc}    
\usepackage{url}            
\usepackage{amsfonts}       
\usepackage{nicefrac}       
\usepackage{microtype}      
\usepackage{xcolor}         

\usepackage{graphicx}
\usepackage{booktabs}
\usepackage{kotex}
\usepackage{bm}

\usepackage{wrapfig}
\usepackage{multirow}
\usepackage{booktabs}
\usepackage{colortbl}
\usepackage{tabulary}
\usepackage{algorithmicx}
\usepackage{algpseudocode}
\usepackage{listings}
\usepackage{multicol}
\usepackage{xspace}
\usepackage{subcaption}
\usepackage{graphbox}
\usepackage{arydshln}
\usepackage{caption}
\usepackage{algpseudocode}
\usepackage{algorithm}
\usepackage{adjustbox}
\usepackage{amsmath}
\usepackage{microtype}

\usepackage{diagbox}

\definecolor{cvprblue}{rgb}{0.21,0.49,0.74}
\usepackage[pagebackref,breaklinks,colorlinks,allcolors=cvprblue]{hyperref}

\def\paperID{N/A} 
\def\confName{CVPR}
\def\confYear{2026}

\title{C$^2$FG: Control Classifier-Free Guidance via Score Discrepancy Analysis}

\author{
Jiayang Gao$^{1,\ast}$, Tianyi Zheng$^{2,\ast}$, Jiayang Zou$^1$, Fengxiang Yang$^2$, Shice Liu$^2$, Luyao Fan$^1$,\\ Zheyu Zhang$^2$, Hao Zhang$^2$, Jinwei Chen$^2$, Peng-Tao Jiang$^2$, Bo Li$^{2,\dagger}$, Jia Wang$^{1,\dagger}$ 
  \\$^1$ {Shanghai Jiao Tong University, Shanghai, China} 
  \\
  $^2${vivo BlueImage Lab, vivo Mobile Communication Co., Ltd., China} \\
  {\tt\small {\{gjy0515, jiawang\}}@sjtu.edu.cn, \{zhengtianyi, libra\}@vivo.com}}

\begin{document}
\maketitle
\footnotetext{$\ast$ These authors contributed equally to this work.}
\footnotetext{$\dagger$ Corresponding authors.}
\begin{abstract}
Classifier-Free Guidance (CFG) is a cornerstone of modern conditional diffusion models, yet its reliance on the fixed or heuristic dynamic guidance weight is predominantly empirical and overlooks the inherent dynamics of the diffusion process. 
In this paper, we provide a rigorous theoretical analysis of the Classifier-Free Guidance. Specifically, we establish strict upper bounds on the score discrepancy between conditional and unconditional distributions at different timesteps based on the diffusion process.
This finding explains the limitations of fixed-weight strategies and establishes a principled foundation for time-dependent guidance. Motivated by this insight, we introduce \textbf{Control Classifier-Free Guidance (C$^2$FG)}, a novel, training-free, and plug-in method that aligns the guidance strength with the diffusion dynamics via an exponential decay control function. Extensive experiments demonstrate that C$^2$FG is effective and broadly applicable across diverse generative tasks, while also exhibiting orthogonality to existing strategies.
\end{abstract}    
\section{Introduction}
\label{sec:intro}

Diffusion models \citep{SohlDickstein2015DeepUL,Song2019GenerativeMB,Song2020DenoisingDI,Song2020ScoreBasedGM} have received widespread attention due to their remarkable generative capabilities and have been successfully applied in image synthesis \citep{rombach2022high,Dhariwal2021DiffusionMB}, speech generation \citep{de2025lipdiffuser}, and 3D generation \citep{woo2024harmonyview}. With the advent of conditional diffusion models, researchers have explored guiding generation using additional information, such as class labels \citep{Dhariwal2021DiffusionMB} or textual descriptions \citep{balaji2022ediff,ramesh2022hierarchical}. Among these, classifier-free guidance (CFG) \citep{Ho2022ClassifierFreeDG} has emerged as a popular approach to improve sample quality. How to effectively incorporate conditional information remains a central challenge in conditional diffusion model design.

Diffusion models \citep{SohlDickstein2015DeepUL,Song2019GenerativeMB,Song2020DenoisingDI,Song2020ScoreBasedGM,11143960, shen2025information} 
are grounded in the principle of gradually transforming noise into data through a reverse denoising process, where conditional generation requires effective mechanisms for incorporating guidance.
Most of the conditional diffusion models are based on Bayes' theory: 
Early approaches such as Classifier Guidance (CG) \citep{Dhariwal2021DiffusionMB}  introduced an auxiliary classifier to steer the sampling trajectory toward the target condition. 
While effective, this approach is often unstable and relies on training an additional classifier, which can be difficult and computationally expensive \citep{vaeth2024gradcheck}. 
To address these limitations, Classifier-Free Guidance (CFG) was proposed as a more practical solution, enabling conditional generation without the need for an external classifier. The key motivation of CFG lies in its ability to interpolate between unconditional and conditional score estimates, thus providing a flexible and straightforward mechanism for conditional control.

Despite its success, the original design of CFG fixes guidance in the time domain, which may not be optimal. Subsequent works have extended CFG by exploring alternative strategies:  Interval Guidance~\cite{kynkaanniemi2024applying} propose restricting guidance to a limited interval of noise levels, FDG~\cite{Sadat2025GuidanceIT} choose a low cfg-scale for low frequencies and a high cfg-scale for high frequencies,  CFG++~\cite{chung2025cfg} and TFG~\cite{ye2024tfg} constrain classifier-free guidance to the data manifold, $\beta$-CFG~\cite{malarz2025classifier} and RAAG~\cite{zhu2025raagratioawareadaptive} adjust guidance strength via a time-dependent distribution. While these efforts deepen the community’s understanding and lead to tangible improvements in generative performance, they remain largely heuristic and motivated by empirical observations rather than rigorous theory. \textit{More importantly, they often overlook a fundamental aspect of CFG’s design: the inherent difference between the conditional and unconditional data distributions. Consequently, these methods remain sub-optimal, lacking the principled and theoretically-grounded solutions necessary to combine the conditional and unconditional scores effectively across different diffusion stages.}

In this paper, we aim to provide a theoretical understanding of the difference between conditional and unconditional outputs in classifier-free guidance.
Specifically, we analyze the problem from the perspective of differences between score functions of conditional and unconditional distributions in Theorems~\ref{thm mse vp} and \ref{thm mse ve}. Moreover, we explore the relationship of distributions at different timesteps and locations in Theorems~\ref{VPthm} and \ref{VEthm}. 
Theoretically, Theorems~\ref{thm mse vp} and \ref{thm mse ve} establish rigorous bounds on the score discrepancies, which in turn reveal intrinsic limitations in existing approaches that rely on unconditional guidance alone. Furthermore, we empirically validate our theoretical findings (Figure~\ref{fig:MSE and Cosine}), showing that the derived upper bounds on score MSE hold in practice. Besides, Theorems~\ref{VPthm} and \ref{VEthm} show that it's hard to bound the probability density function (PDF) when the timestep tends to 0. By integrating these theoretical and empirical insights, we confirm that the difference between conditional and unconditional outputs is strictly monotonically decreasing in the forward process. \textit{This insight inspires us to design a time-decaying weighting for CFG, which optimally balances the unconditional and conditional guidance throughout the generation process, thereby enhancing generation quality.}

Building on our theoretical analysis and empirical validations, we propose Control Classifier-Free Guidance (C$^2$FG), a novel guidance strategy in conditional diffusion models. The key design of our approach is to replace the fixed guidance weight with a time-dependent control function, which aligns strictly with our theoretical conclusions. Meanwhile, our method offers greater controllability and provides more flexible choices for balancing fidelity and diversity. Importantly, it is a training-free approach, requiring no additional classifier training, and can be seamlessly applied to a wide range of advanced diffusion frameworks, such as Stable Diffusion (SD) \citep{rombach2022high}, EDM2 \citep{Karras2024edm2}, U-ViT \citep{bao2022all}, DiT \citep{Peebles2022DiT}, and SiT \citep{ma2024sit}. Besides, C$^2$FG is not only generalizable across various generative tasks but also orthogonal to existing strategies (e.g., autoguidance~\citep{Karras2024autoguidance}). Moreover, our approach can theoretically explain and integrate the interval guidance strategy \cite{kynkaanniemi2024applying} and can be seamlessly applied to the exceptionally strong SiT-XL/2 (REPA), a strong baseline already difficult to improve. Even in this setting, C$^2$FG yields further gains in both FID and IS scores while maintaining other metrics.
Overall, our main contributions can be summarized as:
\begin{enumerate}
    \item \textbf{Theoretical analysis:} We provide a rigorous theoretical analysis of the discrepancy in CFG, revealing that the difference between conditional and unconditional scores dynamically decays over time. This insight establishes a principled foundation for time-dependent scaling and exposes the fundamental limitations of a fixed guidance weight.
    \item \textbf{Method design:} Guided by our analysis, we propose \textbf{Control Classifier-Free Guidance (C$^2$FG)}, a theoretically-grounded, training-free method that implements a time-dependent exponential decay control function. This design enhances controllability over the generation process by aligning guidance strength with the underlying diffusion dynamics.
    \item \textbf{Experimental validation:} We demonstrate that C$^2$FG achieves SOTA performance across various conditional generation benchmarks. Moreover, C$^2$FG can be applied to various sampling designs, including stochastic and ordinary differential equations, demonstrating its versatility.
    Notably, C$^2$FG's orthogonal design enhances even exceptionally strong baselines like SiT-XL/2 (REPA) with interval guidance \citep{kynkaanniemi2024applying}. It provides further gains in FID and IS scores while maintaining other key metrics.

\end{enumerate}

\section{Background}
\label{background}

\subsection{Diffusion Models}
Diffusion models  \citep{SohlDickstein2015DeepUL,Song2019GenerativeMB,Song2020DenoisingDI,Song2020ScoreBasedGM} learn complex data distributions through a two-stage procedure. The forward process adds noise to the data step by step, while the reverse process removes the noise to recover the target distribution.




\noindent \textbf{Stochastic Differential Equations (SDEs).} Mathematically, diffusion models can be described using stochastic differential equations. The forward process can be expressed as an SDE \cite{Song2020ScoreBasedGM}:
\begin{align}\label{gen_sde}
    \mathrm{d} x_t = f(x_t, t) \, \mathrm{d}t + g(t) \, \mathrm{d}w_t,
\end{align}
where $f(x_t, t)$ is the drift coefficient, $g(t)$ is the diffusion coefficient, and $w_t$ is a standard Wiener process.
Meanwhile, the corresponding reverse-time SDE \citep{anderson1982reverse} is:
\begin{align}
    \mathrm{d} x_t = &\left[ f(x_t, t) - \frac{1}{2} \left( g^2(t) + \sigma^2(t) \right) \nabla_{x_t} \log p(x_t, t) \right] \mathrm{d}t \nonumber\\&~~~~~~~~~~~~~~~~~~~~~~~~~~~~~~~~~~~~~~~~~~~~~~~~~~~~\quad + \sigma(t) \, \mathrm{d}\bar{w}_t,
\end{align}
where $\bar{w}_t$ is a standard reverse-time Wiener process, and $\sigma(t)$ is a user-specified noise scale. Common choices include $\sigma(t) = g(t)$ as in DDPM \citep{Ho2020DenoisingDP} , 
or $\sigma(t) = 0$ as in DDIM \citep{Song2020DenoisingDI} 
and Probability Flow ODEs \citep{Liu2022FlowSA,Lipman2022FlowMF,Liu2022RectifiedFA}.

\noindent  \textbf{Fokker–Planck Equation (FPE).} 
The SDE \eqref{gen_sde} is governed by the FPE \citep{gardiner1985handbook}:
\begin{align}\label{FPE}
    \frac{\partial p(x, t)}{\partial t} = -\nabla_x \cdot \left( f(x, t) p(x, t) \right) + \frac{1}{2} \Delta_x \left( g^2(t) p(x, t) \right),
\end{align}
where $p(x, t)$ denotes the probability density function (PDF) of \eqref{gen_sde} at time $t$. The FPE \eqref{FPE} describes the time evolution of the probability density function (PDF), where the first term (drift term) describes deterministic transport of probability mass, and the second term (diffusion term) models stochastic spreading due to noise.

\subsection{Classifier-Free Guidance}
For conditional diffusion, we incorporate conditioning variables $y$ into the generative process. To remove the need for an external classifier like Classifier Guidance (CG) \citep{Dhariwal2021DiffusionMB}, Classifier-Free Guidance (CFG) \citep{Ho2022ClassifierFreeDG} proposes a method derived from Bayes’ theorem
\begin{align}\label{nabla log bayes equation}
\nabla \log p(y \mid x_t) = \nabla \log p(x_t \mid y) - \nabla \log p(x_t).
\end{align}

Specifically, CFG incorporates conditional information into the denoising network based on Bayes’ theorem, and the generation process is given by
\begin{align}\label{CFG}
\hat{\epsilon}(x_t, t, y) = \omega \left[ \epsilon_\theta(x_t, t, y) - \epsilon_\theta(x_t, t, \varnothing) \right] + \epsilon_\theta(x_t, t, \varnothing),
\end{align}
where $\epsilon_\theta(x_t, t, y)$ is trained with conditional information and $\epsilon_\theta(x_t, t, \varnothing)$ is trained without it. The parameter $\omega$ controls the strength of conditional guidance. In most previous work~\citep{Liu2022FlowSA,Ho2022ClassifierFreeDG}, $\omega$ is fixed during the generation process.
While recent studies~\citep{Sadat2025GuidanceIT,kynkaanniemi2024applying,wang2024analysis,zhu2025raagratioawareadaptive, sadat2025trainingproblemrethinkingclassifierfree} find that a fixed $\omega$ is sub-optimal, the dynamic strategies they propose are often based on heuristic designs and lack clear theoretical guidelines.
Other methods~\cite{jin2025stagewisedynamicsclassifierfreeguidance-mm1, chen2025s2guidancestochasticselfguidance-mm3,bradley2025classifierfree-mm2}, designed for multimodal tasks, suffer from a lack of generality, as they often rely on task-specific properties such as large guidance scales ($\omega$) or external reward models.
Thus, most current approaches are either heuristic or not sufficiently general. We provide a detailed discussion about these works in Appendix \ref{related}.
\section{Method}
\subsection{Theoretic Analysis of Diffusion Process}\label{theory}
In diffusion models, the forward diffusion process is formulated as the Ornstein–Uhlenbeck (OU) process. i.e,
\begin{align}\label{OU-process}
    \rd x_t=f(t)x_t\rd t+g(t)\rd w_t.
\end{align}
Two widely used parameterizations are the Variance-Preserving SDE (VP-SDE) and the Variance-Exploding SDE (VE-SDE) \citep{Song2020ScoreBasedGM}.

\noindent \textbf{VP-SDE.} VP-SDE is designed to keep the marginal variance of $x_t$ bounded during the forward process, typically matching the discrete-time DDPM formulation. A common choice is 
$f(t)=-\frac12 \beta_t,\quad g(t)=\sqrt{\beta_t},$
where $\beta_t$ controls the noise schedule. The process gradually drives $x_t$ towards an isotropic Guassion with fixed variance.
Therefore, \eqref{OU-process} becomes
$\rd x_t=-\frac12 \beta_t x_t\rd t+\sqrt{\beta_t}\rd w_t.$

\noindent \textbf{VE-SDE.} The VE formulation increases the variance of $x_t$ over time, corresponding to a pure diffusion process. A typical parameterization is
$f(t)=0,\quad g(t)=\sqrt{\frac{\rd\sigma_t^2}{\rd t}}, $
leading to
$ \rd x_t=\sqrt{\frac{\rd\sigma_t^2}{\rd t}}\rd w_t.$

The forward diffusion process aims to transform unknown data distributions into predefined ones (e.g., Gaussian). Although the initial distributions under different conditions differ in the early stages, they become increasingly similar as the process progresses. As shown by our mathematical analysis, this convergence is non-uniform, meaning that the rate at which conditional information is lost also varies over time. This property challenges the commonly used fixed guidance strategy, indicating that a constant guidance strength is not consistent with the mathematical properties of diffusion.

To guide the design of a time-dependent weighting function $w(t)$, we consider the mean-square error between the scores of distributions induced by different initial conditions. 
\paragraph{Score MSE Bounds.} 

Denote $\tilde{p}(x_t, t) = p(x_t,t|y)$ for the conditional distribution given $y$, then we aim to estimate upper bound of mean square loss of scores between $p(x_t, t)$ and $\tilde p(x_t,t)$:

\begin{theorem}[VP-SDE Score MSE Bound]\label{thm mse vp}
Assume that the sample space is bounded and closed. Then we consider the VP-SDE
\begin{align}\label{VP-SDE}
    \mathrm{d}x_t = -\frac12 \beta(t) x_t \, \mathrm{d}t + \sqrt{\beta(t)} \, \mathrm{d}w_t,
\end{align}

let $p(x,t)$ and $\tilde p(x,t)$ denote the probability densities at time $t$, induced by initial distributions $p(x_0)$ and $\tilde p(x_0)$, respectively. 

Then, the mean-square error (MSE) between the scores satisfies the uniform bound
\begin{align}\label{VP-SDE score final}
    \|\nabla \log p(x,t) - \nabla \log \tilde p(x,t)\| \le \frac{\alpha(t)}{\sigma^2(t)}\, C,\nonumber\\  \forall x \in \text{supp}, \ t \ge 0, 
\end{align}

where $C$ is a constant, $\alpha(t) = \exp{(-\frac{1}{2}\int_0^t\beta_s\rd s)}$, and $\sigma(t)= \alpha(t)\sqrt{\int_0^t \frac{\beta_s}{\alpha^2(s)} \rd s}$.
\end{theorem}
\begin{proof}
    See in Appendix \ref{proof thm1}.
\end{proof}
To make the result more intuitive, we reparameterize time by $t' = \frac{1}{2} \int_0^t \beta_s \rd s$, under which the SDE becomes
$$\rd x_{t'}=-x_{t'}\rd t'+\sqrt2 \rd w_{t'},$$ 
so that $\alpha(t') = e^{-t'}$ and $\sigma^2(t')=\frac{1-e^{-2t}}{2}$.

By Theorem~\ref{thm mse vp}, we obtain
\begin{align}\label{Heat-SDE score}
\|\nabla \log p(x,t) - \nabla \log p(x,t|y)\| \le \frac{e^{-t}}{1 - e^{-2t}}\, C,
\end{align}
which admits an asymptotic upper bound of order $O(e^{-t})$ decay rate when t is large.

\begin{theorem}[VE-SDE Score MSE Bound]\label{thm mse ve}
Assume that the sample space is bounded and closed. Then we consider the VE-SDE
\begin{align}\label{VE-SDE}
    \rd x_t=\sqrt{\frac{\rd\sigma_t^2}{\rd t}}\rd w_t,
\end{align}

let $p(x,t)$ and $\tilde p(x,t)$ denote the probability densities induced by initial distributions $p(x_0)$ and $\tilde p(x_0)$, respectively. Assume that the sample space is bounded and closed.

Then, the mean-square error (MSE) between the conditional and unconditional scores satisfies the uniform bound
\begin{align}\label{VE-SDE score final}
    \|\nabla \log p(x,t) - \nabla \log \tilde p(x,t)\|
    &\le \frac{C}{\sigma^2(t)}, \nonumber\\
    &\forall\, x \in \text{supp},\ t \ge 0,
\end{align}

where $C$ is a constant.
\end{theorem}
\begin{proof}
    See in Appendix \ref{proof thm2}.
\end{proof}

From Theorems~\ref{thm mse vp} and~\ref{thm mse ve}, we obtain a relatively precise characterization of how the discrepancy between the conditional and unconditional score functions decays for both VP-SDE and VE-SDE. This implies that, as the diffusion process evolves, the effect of conditioning information gradually diminishes. Nevertheless, the theoretical bounds in \eqref{VP-SDE score final} and \eqref{VE-SDE score final} become singular as $t\to 0,$ , which makes them unsuitable for direct use in practical weighting schemes. However, as noted in~\cite{Song2020ScoreBasedGM}, estimating the score function near $t=0$ is also inherently difficult. In practice, we simply disregard this regime, since the function is hard to fit and the score estimates are unreliable. Consequently, on the interval $t\in [t_0,T]$ with $t_0>0$, we instead try to employ a continuous, time-decaying, and non-singular function that provides a uniform upper bound on the score discrepancy. In practical implementations, this surrogate function can be smoothly and approximately extrapolated toward $t\to 0$, allowing for stable and convenient engineering-level hyperparameter tuning.

In addition, we can estimate the bound between $p(x_1, t_1)$ and $p(x_2,t_2)$, $t_1<t_2$ in Theorem \ref{VPthm} and \ref{VEthm}, which claim that when fixing $p(x_2,t_2)$, the upper bound of $p(x_1, t_1)$ becomes larger as $t_1$ tends to 0 and $x_1$ deviates further from $x_2$. Via these estimations, we can take an insight into the difference of PDF when $t\to0$.

\paragraph{Harnack-type PDF Inequalities.}
The probability density functions themselves satisfy the following inequalities, which provide further insight into the evolution of the distributions over time:
\begin{theorem}[Harnack-type Inequality of VP-SDE]\label{VPthm}
    Let $p(x_t,t) \in C^{2, 1}(\mathbb R^n\times[0,+\infty))$ denote the probability density function of the VP-SDE \eqref{VP-SDE}, and define
    
    $$s(t) = \frac12 \int_0^t \beta_r dr, ~~t(s) = s^{-1}(t).$$
    
    Then for any $\alpha >1, x_1, x_2\in \mathbb R^n, 0<s_1<s_2<+\infty$, the following inequality holds:
    \begin{align}\label{VPinequa}
    p(x_1, t(s_1))
    &\leq p(x_2, t(s_2))
    \left(\frac{s_2}{s_1}\right)^{\frac{m\alpha}{2}}
    \nonumber\\
    &~~\times
    \exp\!\left(
        \frac{\alpha^2 \|x_1 - x_2\|^2}{4(s_2 - s_1)}
        + \frac{\|x_2\|^2 - \|x_1\|^2}{2}
    \right),
\end{align}
    where $m\geq n$ and $\|\cdot\|$ denotes the Euclidean distance.
\end{theorem}
\begin{proof}
    See in Appendix \ref{proof thm3}.
\end{proof}
\begin{theorem}[Harnack-type Inequality of VE-SDE]\label{VEthm}
    Similarly, let $p(x_t,t) \in C^{2,1}(\mathbb R^n\times[0,+\infty))$ denote the probability density function of the VE-SDE \eqref{VE-SDE}, and define 
    $$s(t) = \sigma_t^2, ~~t(s) = s^{-1}(t).$$ 
    Then for any $\alpha >1, x_1, x_2\in \mathbb R^n, 0<s_1<s_2<+\infty$, the following inequality holds for $p$:
    \begin{align}\label{VEinequa}
        p(x_1,t(s_1)))&\leq p(x_2,t(s_2)) \left(\frac{s_2}{s_1} \right)^{\frac{n\alpha}{2}}\nonumber\\
    &\quad\times\exp{\left(\frac{\alpha^2\|x_1-x_2\|^2}{4(s_2-s_1)}\right)}.
    \end{align}
\end{theorem}
\begin{proof}
    See in Appendix \ref{proof thm4}.
\end{proof}

We take VP-SDE for example, fix $x_2 = x, s_2=s$, and assume $p(x,t(s))>0$, then
we can see that the upper bound of $p(x_1,t(s_1))$ is increasing as $s_1$ decreases and $d(x_1,x)$ increases. When $t\to 0$, it becomes harder to bound the PDF, which indicates that the  `magnitude' or `amplitude' of the PDF at early times (small $s_1$) can be much larger than at later times. Moreover, the closer we are to the initial time, the greater the diversity of the PDF, which amplifies the differences between different initial distributions. We can obtain similar conclusion from VE-SDE using the same method.

These inequalities complement the score MSE bounds, offering a detailed view of how the densities evolve and spread over time, further supporting the design of the exponentially decaying weighting function $\omega(t)$.

\paragraph{Relationship between Harnack-type Inequalities and MSE Bound.} 
Moreover, Appendix \ref{relationship} provides a deeper insight into the connection between Theorems \ref{thm mse vp} and \ref{VPthm}: they respectively lead to Theorems \ref{KL-bound} and \ref{thm ent-cost} in the Appendix \ref{relationship}. In essence, these two results offer complementary perspectives on the evolution of the KL divergence. Specifically, Theorem \ref{thm ent-cost} provides an $O(1/t)$ upper bound on the KL divergence, while Theorem \ref{KL-bound} establishes an $O(e^{-2t}/(1-e^{-2t}))$ upper bound decay rate. The former offers a sharper characterization in the short-time regime, whereas the latter dominates in the long-time limit. Together, they delineate a coherent picture of how discrepancies dissipate over time.
\subsection{Control Classifier-Free Guidance (C$^2$FG)}

\paragraph{Intuitive Motivation.}  
We established a core fact in Section~\ref{theory} (Theorems~\ref{thm mse vp} and~\ref{thm mse ve}): during the \textit{forward} diffusion process ($t: 0 \to T$), the conditional and unconditional distributions gradually converge, and the uniform upper bound of score discrepancy
\begin{align*}
    \text{discrepancy}(t) \ge \|\nabla \log p(x_t|y) - \nabla \log p(x_t)\|, \nonumber \\
\forall t > 0, x_t\in \text{Supp}, y\in \mathbb N,
\end{align*}
decays as $t$ increases. 
In~\eqref{Heat-SDE score}, after reparameterizing $t$ to $t^\prime$, the theoretical bound  takes the form $\tfrac{e^{-t'}}{1-e^{-2t'}}\cdot C \sim O(e^{-t'})$. This shows that in the reparameterized time scale, the conditional–unconditional score discrepancy decays exponentially.
Since $t'$ is a monotone function of the original time $t$, this result implies that in the practical diffusion clock $t\in [0,T],$ the score discrepancy exhibits an \textbf{approximately exponential trend}, with the precise rate governed by the effective diffusion coefficient $\beta_t$.
Hence, the key insight is that the conditional discrepancy evolves roughly as
$$\text{discrepancy}(t)\propto e^{-t}.$$

This property has direct implications for the \textit{reverse sampling} process ($t: T \to 0$):

\begin{itemize}
    \item \textbf{Early generation (high $t$, near pure noise):} the conditional score $\nabla \log p(x_t|y)$ and unconditional score $\nabla \log p(x_t)$ are \textbf{highly similar}.
    \item \textbf{Late generation (low $t$, near data):} the two score functions \textbf{diverge significantly}.
\end{itemize}

Our empirical validation (Figure~\ref{fig:MSE and Cosine}) strongly supports this theoretical prediction. During reverse sampling, we observe that as $t \to 0$, the mean-square error (MSE) between conditional and unconditional scores (Figure~\ref{fig:mse}) \textbf{grows exponentially}, while the cosine similarity (Figure~\ref{fig:cos}) \textbf{decreases}, indicating divergence in both magnitude and direction.
Moreover, we visualize these discrepancies as heatmaps of the logarithmic ratio between the two score functions on various timesteps (See Appendix \ref{more exp}, Figure \ref{fig:score ratio}).

These trends explain why a fixed guidance factor ($\omega$ constant) cannot fully capture the true diffusion dynamics: it treats intrinsic differences at all timesteps \textit{uniformly}.

\begin{itemize}
    \item \textbf{Early generation (high $t$):} when the score discrepancy is already small, a fixed (or large) $\omega$ may introduce \textbf{unnecessary, excessive guidance}, potentially disrupting the natural structure formation (as discussed in \cite{kynkaanniemi2024applying}).
    \item \textbf{Late generation (low $t$):} As the diffusion process approaches the data manifold, the discrepancy between conditional and unconditional scores reaches its maximum. In this stage, a fixed (or small) $\omega$ may be \textbf{insufficient}, failing to pull the sample trajectory towards the target conditional manifold, thereby reducing fidelity.
\end{itemize}

Therefore, an ideal guidance schedule $\omega(t)$ should \textbf{align the exponential law} that governs the conditional–unconditional score discrepancy:
$$\omega(t) = \lambda e^{-t}\propto \text{discrepancy}(t),$$
under this setting, the guidance schedule $\omega(t)$ should follow a \textbf{exponentially decreasing} trend over time.

Furthermore, the Harnack-type inequalities (Theorems~\ref{VPthm} and~\ref{VEthm}) provide additional support from the probability density function (PDF) perspective. They indicate that as $t \to 0$, the \textit{magnitude} and \textit{diversity} of the PDF become difficult to control (i.e., the upper bound diverges). In this high-discrepancy, high-diversity ``critical region'' ($t \to 0$), a strong guidance signal $\omega(t)$ is necessary to \textit{steer} the generation process, ensuring precise convergence to the target conditional distribution.

Both the score-MSE discrepancy and Harnack-type inequality indicate an exponential upper bound. Therefore, we design an exponentially decreasing $\omega(t)$ that is calibrated to match the exponential upper bound.

\paragraph{Methodology Design.}  
To embody this intuition, we propose the \textbf{Control Classifier-Free Guidance (C$^2$FG)}, which replaces the fixed guidance $\omega$ in standard CFG with a time-varying control function based on previous analysis:
\begin{align}\label{eq:ecfg}
    \omega(t) = \omega_0 \exp\Big(\lambda \Big(1-\frac{t}{t_{\max}}\Big)\Big),
\end{align}
where $t_{\max}$ is the maximum diffusion time determined by the forward process, $\lambda>0$ controls the growth rate, determining how much the guidance increases from $\omega_0$ at $t=t_{\max}$ to $\omega_0 e^{\lambda}$ at $t=0$, directly capturing the theoretically proven exponential growth of score discrepancy during the reverse process.
This design has several benefits:
\begin{enumerate}
    \item \textbf{Consistency with both theory and observation.} Theoretical bound $\frac{\alpha(t)}{\sigma^2(t)}\, C$ in Theorem \ref{thm mse vp} and MSE bound in Figure \ref{fig:mse} suggest exponential-like decay of conditional discrepancy, which aligns naturally with this schedule.  
    \item \textbf{Smoothness and stability.} Unlike step-wise or linear schedules, the exponential function is continuously differentiable, reducing numerical instability.  
    \item \textbf{Easy to tune.} Only two hyperparameters are introduced: $\omega_0$ (maximum guidance strength, same as standard CFG) and $\lambda$ (control decay rate). 
    \item \textbf{Interpretability.} $\omega_0$ sets the  guidance strength, and $\lambda$ directly controls the speed of conditional information decay, providing an intuitive trade-off between \emph{faithfulness} and \emph{diversity}. More detailed analysis see Appendix \ref{more exp}.
\end{enumerate}

\paragraph{Integration into Sampling.}  
At each timestep $t$ during generation, the C$^2$FG update replaces the standard CFG as follows:
\[
\epscfgnc(\x_t) = \epsnullnc(\x_t) + \omega(t) \big[\epscond(\x_t) - \epsnullnc(\x_t)\big].
\]

Furthermore, our framework also provides a theoretical interpretation of the interval-based strategy proposed in~\cite{kynkaanniemi2024applying}.
Specifically, during the \textbf{early generation stage}, the discrepancy between conditional ($\epsilon_c$) and unconditional ($\epsilon_\emptyset$) scores becomes negligible.
Thus, only the conditional network is necessary in this regime, as the crucial information distinguishing between \textit{different} conditions (e.g., $c_1$ vs. $c_2$) \textit{only resides} within this network. 
Accordingly, the method of~\cite{kynkaanniemi2024applying} employs a piecewise-constant guidance schedule: $\omega(t)$ is set to a fixed value $\omega_0>1$ within a selected interval $[t_l,t_h]$, and reverts to $1$ outside this region. Its underlying motivation can in fact be interpreted as a special case of our theoretical framework. Moreover, as shown in Figure~\ref{fig:noise2img}, combining our C$^2$FG with their interval-based strategy further reduces model evaluation overhead, since the guidance is applied only where it is most effective.

\begin{figure*}[!hbt]
    \centering
    \resizebox{0.8\textwidth}{!}{ 
    \begin{minipage}{\textwidth}
        \centering
        \begin{subfigure}[t]{0.44\textwidth}
            \centering
            \includegraphics[width=\linewidth]{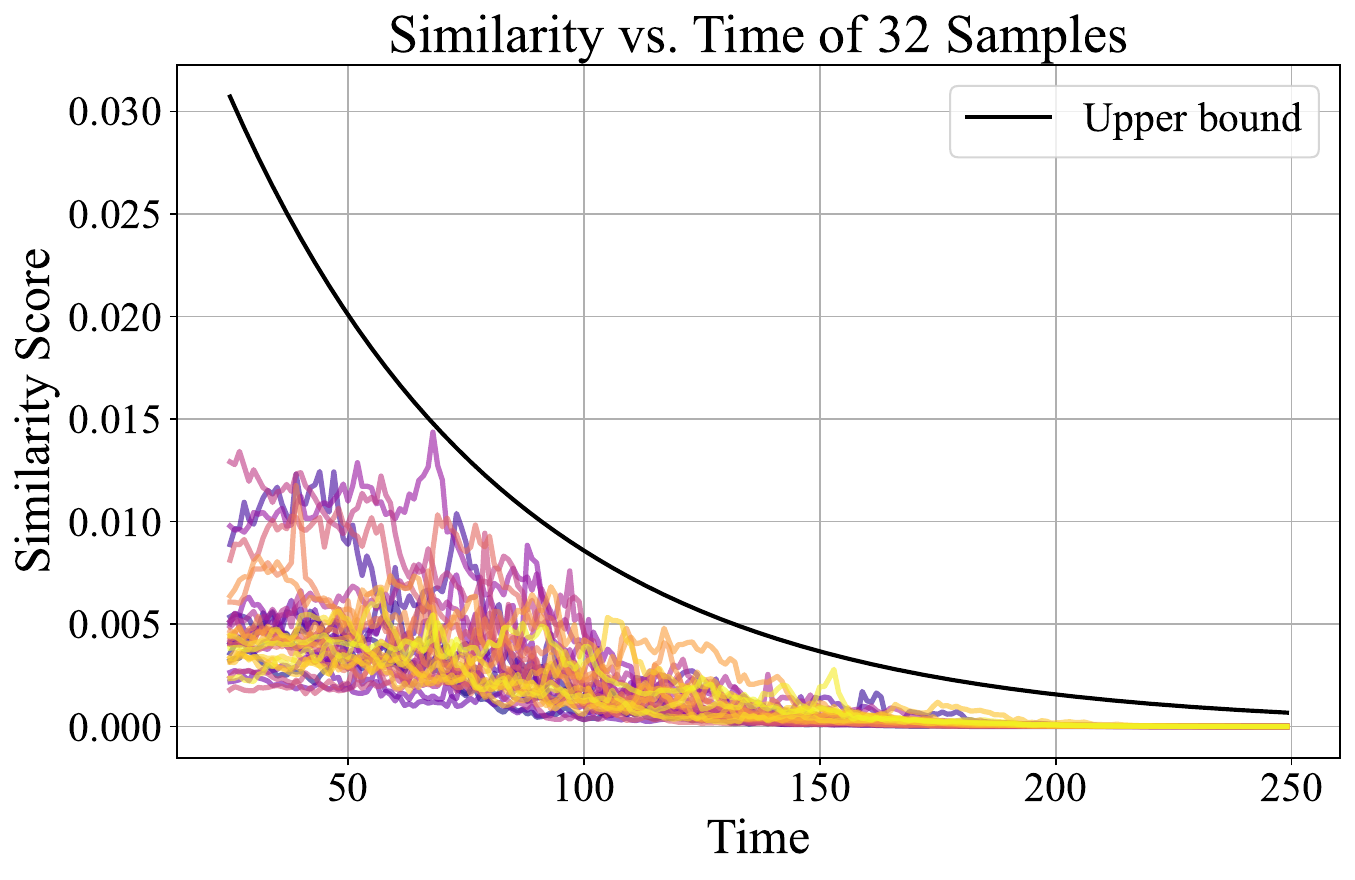}
            \caption{MSE of cond and uncond}
            \label{fig:mse}
        \end{subfigure}
        \hfill 
        \begin{subfigure}[t]{0.42\textwidth}
            \centering
            \includegraphics[width=\linewidth]{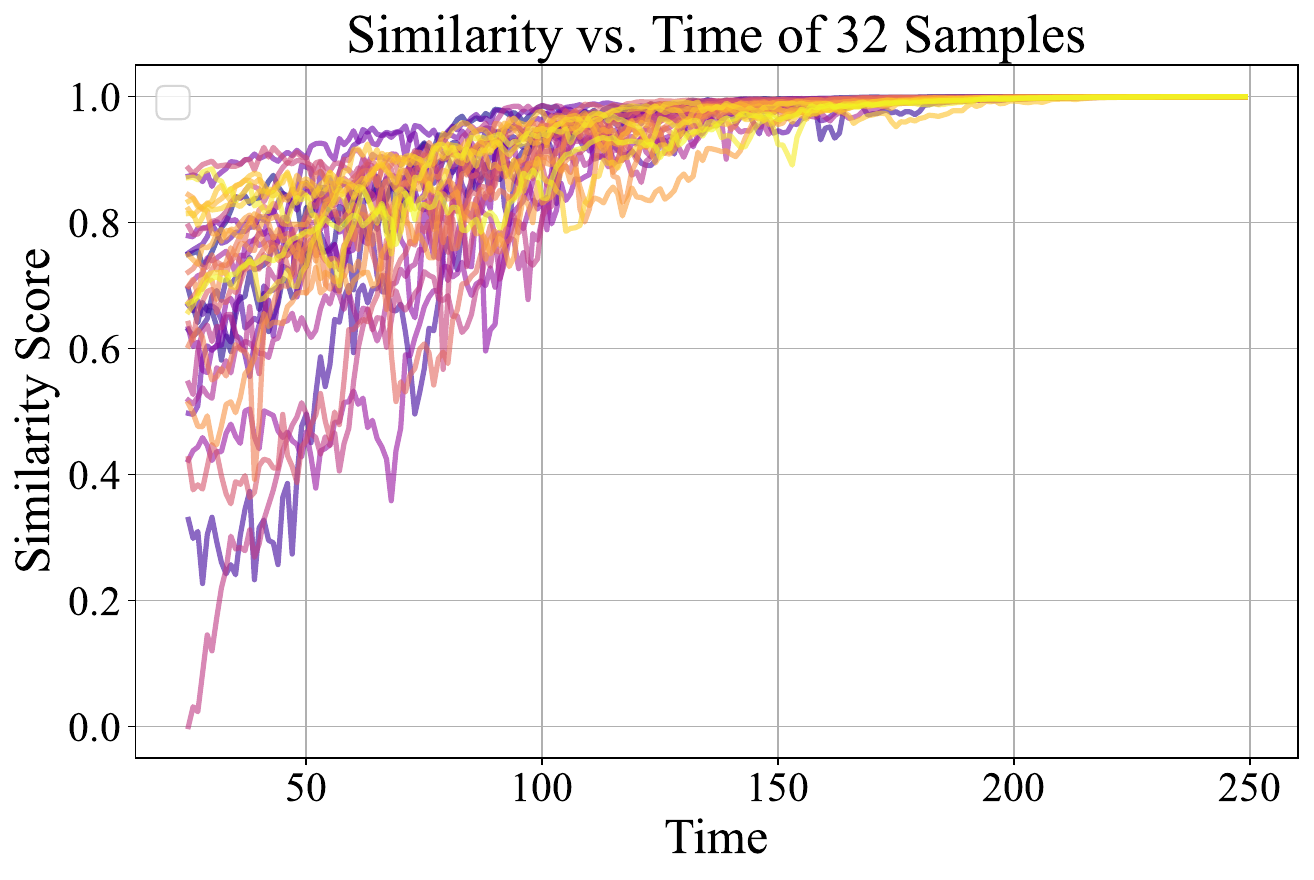}
            \caption{Cosine similarity between cond and uncond}
            \label{fig:cos}
        \end{subfigure}
    \end{minipage}
    }
    \caption{Following \cite{Song2020ScoreBasedGM}, (a) and (b) present results for $t\geq t_0>0$.
    (a) shows that the MSE of conditional score and unconditional score can be bounded by a function which tends to 0 when $t\to+\infty$; 
    (b) shows that the normalized cosine similarity between the two vectors decreases over reverse time, indicating that their directions gradually diverge in the reasoning process.}
    \label{fig:MSE and Cosine}
\end{figure*}

\begin{figure*}[!hbt]
    \centering
    \includegraphics[width=0.8\linewidth]{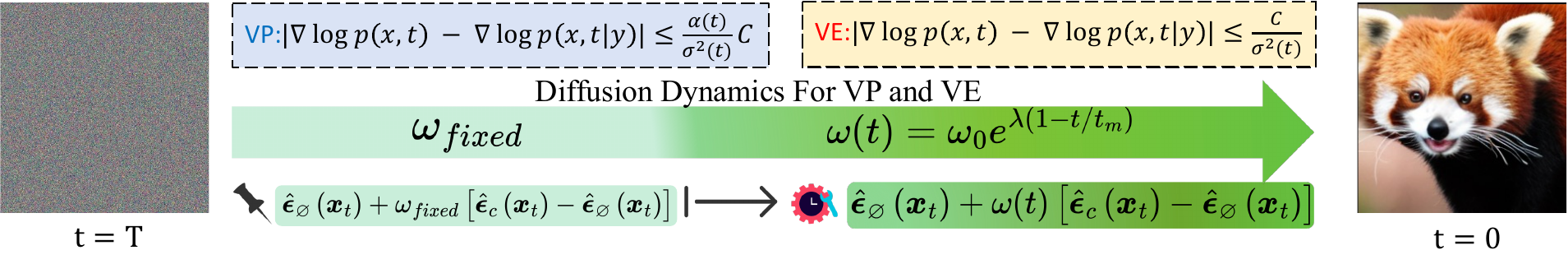}
    \caption{Noise to Image Process of \textbf{C$^2$FG}: Dynamic guidance weight $\omega(t)$ adaptively balances conditional and unconditional outputs at each timestep t during generation, guided by theoretical bounds on the score function. Moreover, we can choose to add the method of \cite{kynkaanniemi2024applying}, where we fix the $\omega(t) =1 $ at the beginning of generation or when $t$ tends to 0.}
        \label{fig:noise2img}
\end{figure*}

\section{Experiments}
\subsection{Experimental Setup}\label{setup}
\noindent \textbf{Models and Datasets.}
We evaluate our method on multiple generative tasks, including conditional image and text-to-image generation. Experiments are conducted on ImageNet \citep{5206848} and MS-COCO \citep{lin2014microsoft} text-to-image datasets. All models are based on advanced diffusion backbones, including U-ViT \citep{bao2022all}, DiT \citep{Peebles2022DiT}, Stable Diffusion \citep{rombach2022high} and SiT \citep{ma2024sit}, using pre-trained weights where applicable.   

\noindent \textbf{Evaluation Metrics.} Quantitative evaluation uses FID \citep{heusel2017gans}, IS \citep{salimans2016improved}, and Precision/Recall \citep{kynkaanniemi2019improved} score to assess both fidelity and conditional alignment. All experiments are implemented in PyTorch and TensorFlow and run on NVIDIA 4090 GPUs. 

\subsection{Experimental Result}
\begin{figure*}[!hbt]
    \centering
    \includegraphics[width=0.7\linewidth]{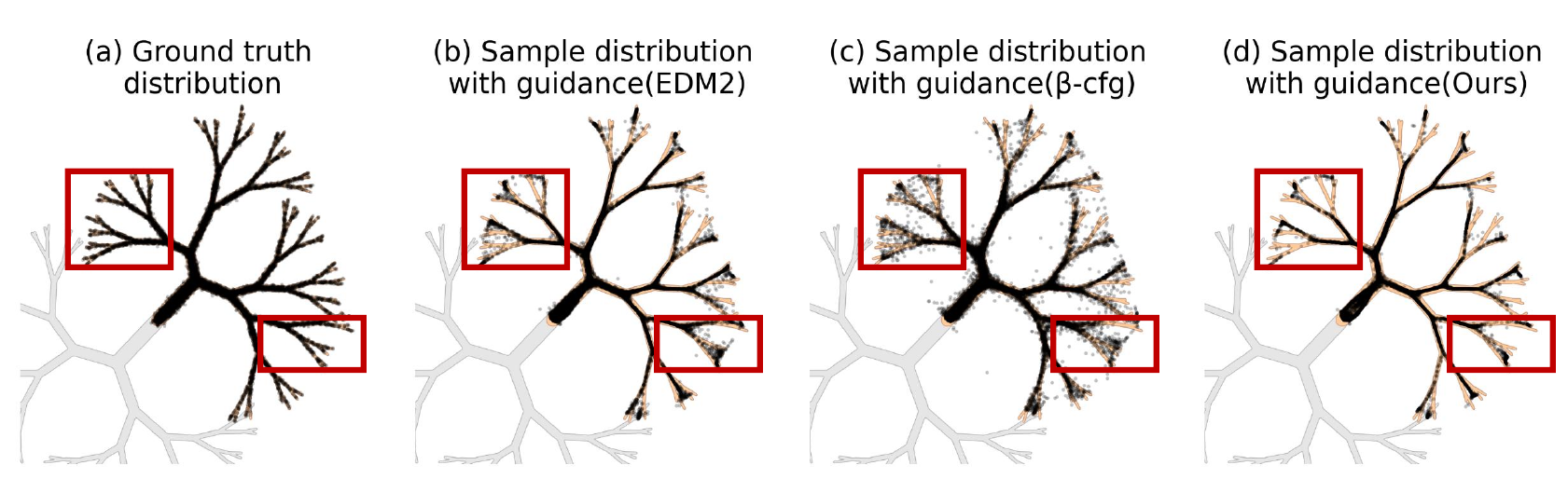}
    \caption{ A two-dimensional distribution featuring two classes
 represented by gray and orange regions. Approximately 99\% of the probability mass is inside the shown contours. (a) Ground truth samples from the orange class. 
        (b) EDM2 ($\omega=1$) produces some outliers. 
        (c) $\beta$-CFG ($\alpha=\beta=2, \omega=1$) produces more outliers. 
        (d) C$^2$FG ($\omega_0=1, \lambda=0.6$) generates fewer outliers and better matches the target distribution.}
    \label{fig:toy}
\end{figure*}
\noindent \textbf{Toy Example.} Figure \ref{fig:toy} presents a 2D toy example comparing three conditional sampling methods: EDM2 \citep{Karras2024edm2}, $\beta$-CFG \citep{malarz2025classifier}, and our proposed method. 
The figure demonstrates that our method produces a more adaptive weighting strategy, resulting in fewer outliers and better alignment with the target distribution compared to the baselines.



\noindent \textbf{Results on DiT.}
In Table \ref{tab:fid}, we quantitatively evaluate our C$^2$FG on the different ImageNet (256 $\times$ 256, 512 $\times$ 512, class-conditional) benchmarks based on DiT diffusion architectures. 
We also compare with the recent SOTA Rectified Diffusion \citep{wang2024rectified} methods. As shown in the Table \ref{tab:fid}, C$^2$FG shows comprehensive improvements across all metrics, exhibiting particularly significant gains in FID and IS scores. Additionally, C$^2$FG is validated on the higher-resolution ImageNet-512 dataset, demonstrating that it remains effective for high-resolution data.

\noindent \textbf{Results on SiT.}
For the SiT baselines, we utilize REPA  as our pre-trained guidance model. During training, REPA aligns the noisy latent features of the diffusion model with representations from pre-trained visual encoders (e.g., MAE \citep{MaskedAutoencoders2021}, DINO \citep{caron2021emerging}), thereby enhancing its generative capabilities. At inference, REPA employs both full \citep{yu2025repa} and interval \citep{kynkaanniemi2024applying} guidance strategies within the timestep range  $(t_l,t_h)$. We evaluate both strategies for a fair comparison. As shown in Table \ref{tab:fid}, our proposed C$^2$FG achieves comprehensive performance improvements at no additional overhead. This effectiveness extends to evaluations using ODE samplers, where C$^2$FG also boosts model performance. Collectively, these results demonstrate the effectiveness of C$^2$FG and its robustness across different samplers.

\begin{table*}[!hbt]
\caption{\textbf{Quantitative Comparison.} Comparison of different evaluation metrics on Class-Conditional ImageNet datasets with different diffusion architectures and inference steps.} 
    \begin{center}
    \scalebox{0.67}
    { 
    \begin{tabular}{lcccccc}
    \toprule
    \multicolumn{7}{l}{{\bf\normalsize ImageNet} } \\
    \toprule
    \multirow{1}{*}{\bf{Model ( 256$\times$256 ), 50k samples, 250 inference timesteps}} 
    & FID$\downarrow$ & IS$\uparrow$ & sFID $\downarrow$ & Prec$\uparrow$ & Rec$\uparrow$ \\
    \arrayrulecolor{black}\midrule 
    DiT-XL/2 ($\omega=1.5$, ODE sampler)     & 2.29 & 276.8 & 4.6 & 0.83 & 0.57  \\
    DiT-XL/2 (Rectified Diffusion, $\omega=1.5$, ODE)     & 2.13 & / & / & 0.83 & 0.58  \\
    \textbf{DiT-XL/2 + Ours}($\omega_0= 1, \lambda=\ln 2$, ODE)       & \textbf{2.07} & \textbf{291.5} & 4.6 & 0.83 & \textbf{0.59} \\
    \arrayrulecolor{gray}\cmidrule(lr){1-7}
    SiT-XL/2 (REPA)($\omega=1.35$, SDE)     & 1.80 & 284.0 & 4.5 & 0.81 & 0.61  \\
    \textbf{SiT-XL/2 (REPA) + Ours} ($\omega_0=1, \lambda=1$, SDE)     & \textbf{1.51} & \textbf{315.0} & 4.6 & 0.80 & \textbf{0.62}  \\
    \arrayrulecolor{gray}\cmidrule(lr){1-7}
    SiT-XL/2 (REPA, Interval) ($\omega=1.8, t_l=0,t_h=0.7$, SDE)     & 1.42 & 305.7 & 4.7 & 0.80 & 0.65  \\
    \textbf{SiT-XL/2 (REPA, Interval) + Ours} ($\omega_0=1.8, \lambda=0.03$, SDE)     & \textbf{1.41} & \textbf{308.0} & 4.7 & 0.80 & 0.65  \\

    \arrayrulecolor{gray}\cmidrule(lr){1-7}
    SiT-XL/2 (REPA)($\omega=1.8$, ODE)     & 3.64 & 366.0 & 4.9 & 0.86 & 0.54  \\
    \textbf{SiT-XL/2 (REPA)+Ours}($\omega_0=1.7,\lambda=0.15$, ODE)    &\textbf{3.40} & 364.2 & \textbf{4.7} & 0.86 & \textbf{0.55} \\
    \arrayrulecolor{gray}\cmidrule(lr){1-7}
    SiT-XL/2 (REPA, Interval) ($\omega=1.8, t_l=0,t_h=0.7$, ODE)     & 1.56 & 283.1 & 4.6 & 0.78 & 0.66  \\
    \textbf{SiT-XL/2 (REPA, Interval) + Ours} ($\omega_0=1.8, \lambda=0.03$, ODE)     & \textbf{1.54} & \textbf{286.0} & 4.6 & 0.78 & 0.66  \\
    \arrayrulecolor{black}\midrule 
    \multirow{1}{*}{\bf{Model ( 512$\times$512 ), 10k samples, 100 inference timesteps}}
    \\
    \arrayrulecolor{black}\midrule 
    DiT-XL/2 ($\omega=1.5$, SDE)     & 6.81 & 229.5 & 20.0 & 0.82 & 0.62  \\
    \textbf{DiT-XL/2 + Ours}($\omega_0= 1, \lambda=\ln 2$, SDE)       & \textbf{6.54} & \textbf{280.9} & \textbf{19.7} & \textbf{0.83} & 0.60 \\
    
    \arrayrulecolor{black}\bottomrule
    \end{tabular}
    }
    
    \end{center}
\label{tab:fid}
\end{table*}

\noindent \textbf{Results on other models and datasets.}
To further validate the generality of C$^2$FG, we further extend our evaluation to text-to-image generation in Table \ref{tab:t2icoco}, another representative conditional generation task. On MS-COCO, we validate the effectiveness of C$^2$FG on both U-ViT \citep{bao2022all} and Stable Diffusion 1.5 \citep{rombach2022high}, as reported in Table~\ref{tab:t2icoco}. Our method consistently improves performance across architectures, lowering the FID of U-ViT from 5.37 to 5.28, and achieving a gain in CLIP-Score on Stable Diffusion.

\begin{table}[!b]
\caption{Evaluation of C$^2$FG on MS-COCO (U-ViT) in latent space and ImageNet-64 (EDM2) in pixel space.}
\centering
\scalebox{0.9}{
\begin{tabular}{lc}
    \toprule
    \multicolumn{2}{l}{\bf Latent Space (MS-COCO)} \\
    \midrule
    Model & FID$\downarrow$ \\
    \midrule
    U-ViT($\omega=2$) & 5.37 \\
    \textbf{U-ViT+Ours}($\omega_0=2,\lambda=0.2$) & \textbf{5.28} \\
    \midrule
    Model & CLIP$\uparrow$ \\
    \midrule
    SD15($\omega=5$) & 31.8 \\
    \textbf{SD15+Ours}($\omega_0=5,\lambda=0.2$) & \textbf{31.9} \\
    \midrule
    \multicolumn{2}{l}{\bf Pixel Space (ImageNet-64)} \\
    \midrule
    Model & FID$\downarrow$\\
    \midrule
    EDM2-S(no autoguidance) & 1.58\\
    EDM2-S-autog($\omega=1.7$) & 1.04\\
    \textbf{EDM2-S-autog+Ours}($\omega_0=1.7,\lambda=0.05$) & \textbf{1.03} \\
    \bottomrule
\end{tabular}
}

\label{tab:t2icoco}
\end{table}
In addition, we test C$^2$FG on ImageNet-64 by applying it to another exceptionally strong baseline: the EDM2~\cite{Karras2024edm2} combined with autoguidance~\citep{Karras2024autoguidance}, where the model operates directly in the pixel domain rather than a latent space. Notably, EDM2-S with autoguidance already achieves an exceptionally strong FID of 1.04, representing a near-saturation performance for pixel-space diffusion models. Remarkably, our C$^2$FG further reduces this number to 1.03. These results highlight that C$^2$FG serves as a plug-and-play design. More results and visualized analyses are provided in Appendix~\ref{more exp}.

\begin{figure*}[!htb]
    \centering
    \includegraphics[width=0.7\linewidth]{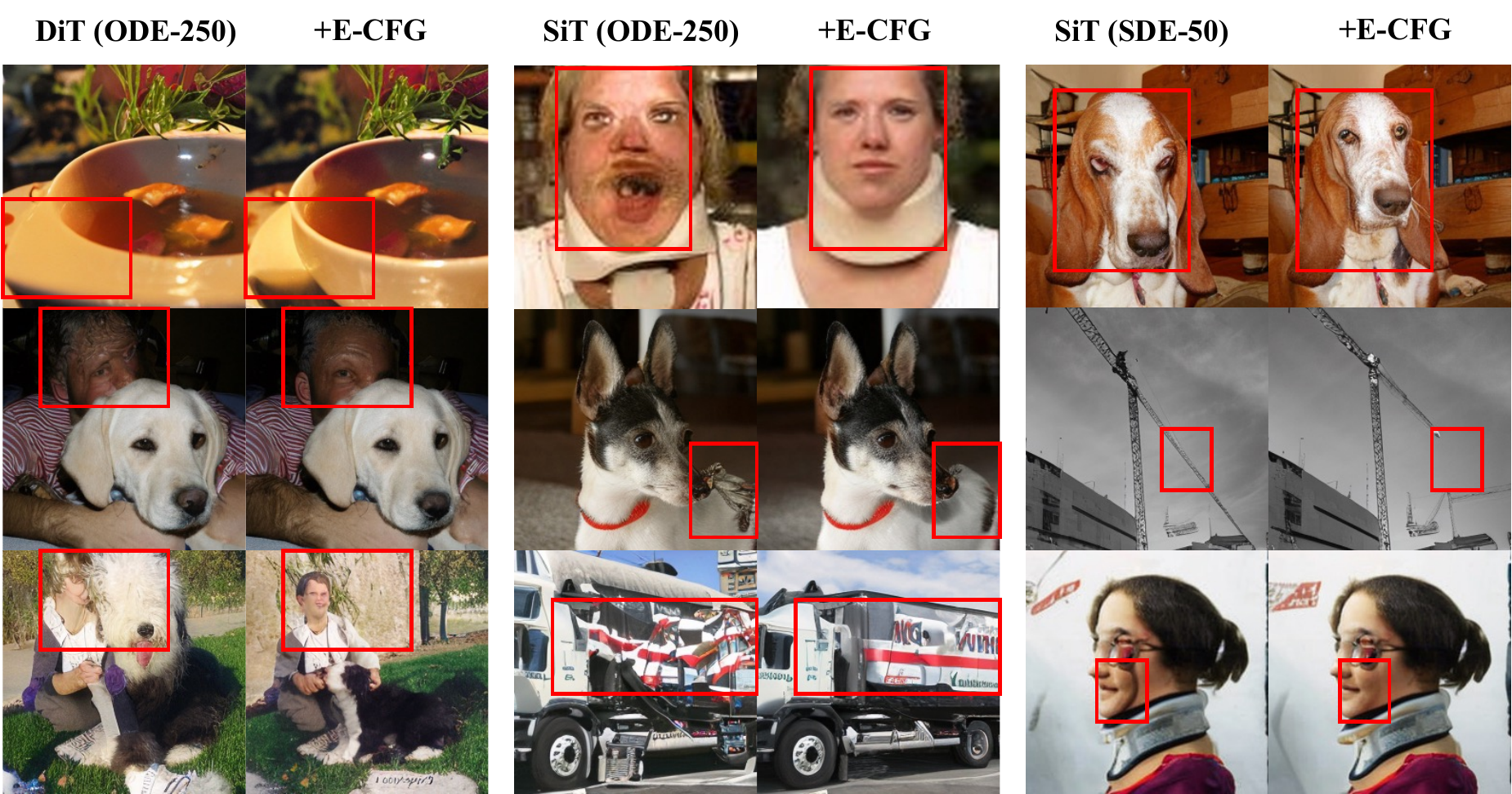}
    
    \caption{\textbf{Qualitative Comparison.} Qualitative comparison on Class-Conditional ImageNet datasets with different architectures and samplers. The sampler used and the number of inference steps are indicated in parentheses. }
    \label{fig:qualitive_multi}
\end{figure*}

\subsection{More Analysis}\label{ablation}
\noindent \textbf{Robustness of the Sampler.} As shown in Table \ref{tab:ablation}, integrating C$^2$FG with SiT-XL/2 (REPA) yields consistent performance gains across various timesteps and sampling schemes. At 50 inference steps, the FID score improves from 3.36 to 3.20 for SDE sampling and from 3.46 to 3.25 for ODE sampling. These improvements become even more pronounced at 20 steps, particularly with the ODE sampler. This indicates that C$^2$FG is robust to the samplers.

\begin{table*}[!hbt]
\caption{\textbf{Ablation Comparison.} Comparison of different evaluation metrics on Class-Conditional ImageNet datasets with different architectures and fewer timesteps.} 
    \begin{center}
    \scalebox{0.75}
    { 
    \begin{tabular}{lcccccc}
    \toprule
    \multicolumn{7}{l}{{\bf\normalsize ImageNet( 256$\times$256 )} } \\
    \toprule
    \multirow{1}{*}{\bf{Model~~\textcolor{red}{50} inference timesteps}} 
    & FID$\downarrow$ & sFID $\downarrow$ & Prec$\uparrow$ & Rec$\uparrow$ \\
    \arrayrulecolor{black}\midrule 
    SiT-XL/2 (REPA)($\omega=1.8$, SDE)     & 3.36 & 4.5 & 0.86 & 0.54  \\
    \textbf{SiT-XL/2 (REPA) + Ours} ($\omega_0=1.7, \lambda=0.15$, SDE)     & \textbf{3.20} & 4.6 & 0.86 & 0.54  \\
    \arrayrulecolor{gray}\cmidrule(lr){1-7}
    
    SiT-XL/2 (REPA)($\omega=1.8$, ODE)     & 3.46 & 4.5 & 0.86 & 0.54  \\
    \textbf{SiT-XL/2 (REPA)+Ours}($\omega_0=1.7,\lambda=0.15$, ODE)    &\textbf{3.25} & \textbf{4.4} & 0.86 & \textbf{0.55} \\
    \arrayrulecolor{black}\midrule 
    \multirow{1}{*}{\bf{Model~~\textcolor{red}{20} inference timesteps}}
    \\
    \arrayrulecolor{black}\midrule 
    SiT-XL/2 (REPA)($\omega=1.8$, SDE)     & 4.38 & 11.8 & 0.79 & 0.53  \\
    \textbf{SiT-XL/2 (REPA) + Ours} ($\omega_0=1.7, \lambda=0.15$, SDE)     & \textbf{4.30} & 12.1 & 0.79 & \textbf{0.54}  \\
    \arrayrulecolor{gray}\cmidrule(lr){1-7}
    
    SiT-XL/2 (REPA)($\omega=1.8$, ODE)     & 3.29 & 4.6 & 0.85 & 0.54  \\
    \textbf{SiT-XL/2 (REPA)+Ours}($\omega_0=1.7,\lambda=0.15$, ODE)    &\textbf{3.10} & \textbf{4.5} & 0.85 & 0.54 \\
    
    \arrayrulecolor{black}\bottomrule
    \end{tabular}
    }
    
    \end{center}
\label{tab:ablation}
\end{table*}

\noindent \textbf{Qualitative Comparison.}\label{qualitative}
Figure \ref{fig:qualitive_multi} presents a qualitative comparison. The examples highlighted in the red box show that C$^2$FG significantly enhances generation quality. Specifically, samples generated by C$^2$FG can effectively mitigate issues such as distortion and blurred texture in generated images.
Moreover, this improvement remains consistent across various samplers and sampling steps, demonstrating the effectiveness and generalizability of C$^2$FG.

\section{Conclusion}
In this work, we theoretically analyze Classifier-Free Guidance by bounding the score discrepancy, revealing the limitations of fixed-weight strategies and motivating time-dependent scaling. Based on this, we propose C$^2$FG, a training-free method that adapts guidance strength via an exponential control function. C$^2$FG consistently improves controllability and achieves state-of-the-art performance across diverse diffusion frameworks, tasks, and inference strategies. Our framework enables principled guidance design and may inspire theoretically grounded methods for conditional diffusion models.

\paragraph{Acknowledgment.} This work was supported by the National Natural Science Foundation of China (NSFC) under Grant U24A20220, by the Shanghai Municipal Commission of Science and Technology under Grant 22DZ2229005, and by the Shanghai Municipal Commission of Economy and Informatization under Grant 2024-GZL-RGZN-01008.

{
    \small
    \bibliographystyle{ieeenat_fullname}
    \bibliography{main}
}
\newpage
\appendix

\onecolumn

\section*{Appendix Overview}
This appendix provides additional details and supplementary results to support the main paper.
In Section~\ref{related}, we review related literature to place our work in a broader context.
Section~\ref{proof of thms} presents the detailed proofs of the theoretical results introduced in the main text.
In Section~\ref{relationship}, we further explore the connection between our MSE bound and Harnack-type inequalities, highlighting their theoretical implications.
In Section \ref{sec alg}, we show the difference between standard CFG and our method.
Finally, Section~\ref{more exp} reports additional experimental results and visualizations.

\section{Related Work}\label{related}
A scaling factor for conditional diffusion models was first introduced in CG \citep{Dhariwal2021DiffusionMB}, which controls the trade-off between fidelity and diversity:
\begin{align}\label{CG}
    \hat{\mu} = \mu_{\theta,\text{uncond}} + \gamma \, \Sigma_\theta(x_t, t) \nabla \log p_t(y \mid x_t),
\end{align}
where $\mu_{\theta,\text{uncond}}$ denotes the predicted mean of the unconditional denoiser, $\Sigma_\theta(x_t, t)$ is the predicted covariance (or noise scale) at step $t$, and $\nabla \log p_t(y \mid x_t)$ represents the conditional score function with respect to the label $y$. The hyperparameter $\gamma$ is the classifier-free guidance scale: $\gamma>1$ strengthens conditioning at the cost of diversity, while $\gamma<1$ weakens conditioning but increases sample diversity. This scaling modifies the reverse sampling distribution as:
\begin{align}
    \tilde{p}(x_{t-1} \mid x_t, y) = \frac{p(x_{t-1} \mid x_t) \, p^\gamma(y \mid x_{t-1})}{Z(x_t, y)}, 
    \quad Z(x_t, y) = \sum_{x_{t-1}} p(x_{t-1} \mid x_t) \, p^\gamma(y \mid x_{t-1}).
\end{align}

Then CFG \citep{Ho2022ClassifierFreeDG} eliminates the need for an external classifier by jointly training the network for both conditional and unconditional predictions:
\begin{align}
    &\epsilon_\theta(x_t, t, y) = \mu_{\theta,\text{uncond}} + \Sigma_\theta(x_t, t) \nabla \log p_t(y \mid x_t), 
    \quad\epsilon_\theta(x_t, t, \phi) = \mu_{\theta,\text{uncond}}.
\end{align}
where $\epsilon_\theta$ is the neural network’s output for noise prediction. 
Substituting these into \eqref{CG} and setting $\gamma = \omega$ yields the CFG formulation:
\begin{align}
    \hat{\epsilon}(x_t, t, y) = \omega \left[ \epsilon_\theta(x_t, t, y) - \epsilon_\theta(x_t, t, \phi) \right] + \epsilon_\theta(x_t, t, \varnothing),
\end{align}
We see that CFG and CG are using the same scaling factor. And for now CFG with this scaling technique that has been widely adopted in mainstream diffusion models, typically with a fixed CFG-scale.

However, recent studies have pointed out that using a constant guidance weight is not necessarily optimal and may lead to limitations in balancing fidelity and diversity. Specifically, several works have proposed various forms of dynamic or time-dependent scaling strategies to improve generation quality. 
\cite{Sadat2025GuidanceIT} proposes Frequency-Decoupled Guidance (FDG), an improved version of classifier-free guidance that operates in the frequency domain, which chooses a low cfg-scale for low frequencies and a high cfg-scale for high frequencies.
\cite{kynkaanniemi2024applying} observe that applying a constant classifier-free guidance (CFG) weight across all noise levels is suboptimal: guidance harms diversity in the high-noise regime, has little effect in the low-noise regime, and is only beneficial in the middle. They propose restricting guidance to a limited interval of noise levels, which both improves sample fidelity and diversity while reducing computational cost. 
\cite{poleski2025geoguide} propose a geometric guidance method for CG to address the vanishing gradient issue in late denoising stages of probabilistic approaches. Its core innovation enforces fixed-length gradient updates ($\|\nabla p\|$-normalized) proportional to data dimension ($\sqrt{D}/T$), maintaining consistent guidance strength throughout sampling.
\cite{lin2024common} rescale classifier-free guidance to prevent
over-exposure.
\cite{malarz2025classifier} propose $\beta$-adaptive scaling to address the trade-off between image quality and prompt alignment in standard CFG. It dynamically adjusts guidance strength via a time-dependent $\beta$-distribution $\beta(t)$, enforcing weak guidance at initial/final steps and strong guidance during critical mid-denoising phases. 
\cite{wang2024analysis} investigate different time-dependent schedulers for the guidance weight. Their analysis and experiments confirm that dynamic weighting strategies outperform fixed weights, with high weights being beneficial in the mid-noise regime but detrimental at the extremes. 
\cite{chung2025cfg} and \cite{chen2025diffieroptimizingdiffusionmodels} improve diffusion model performance by constraining CFG to the data manifold, enabling higher-quality generation, better inversion, and smoother interpolation at lower guidance scales.
\cite{shen2024rethinking} mitigate spatial inconsistency in classifier-free guidance by introducing Semantic-aware CFG, which segments latent images into semantic regions via attention maps and adaptively assigns region-specific guidance scales, leading to more balanced semantics and higher-quality generations.
\cite{wang2025diffusion} propose Diffusion-NPO, which incorporates non-parametric optimization into diffusion sampling via nearest-neighbor matching, improving sample diversity and quality without retraining and working across different models and datasets. 
\cite{bradley2025classifierfree-mm2}  investigate PCG, a theoretical foundations of CFG, and reveal the differrence and relationship between DDPM and DDIM, embedding CFG in a broader design space of principled sampling methods.
\cite{ye2024tfg} introduce TFG, a framework which encompasses existing methods as special cases. In their framework, they defined a  hyper-parameter space for their algorithm and analyze the underlying theoretical motivation of each hyper-parameter.

As for recent work, \textbf{RAAG} \cite{zhu2025raagratioawareadaptive} recompute $\omega$ at every reverse step via a lightweight exponential map of the current RATIO: $\omega(\rho) = 1+(\omega_{\max}-1)\exp(-\alpha \rho)$, which is similar to the form of our C$^2$FG. However, RAAG is primarily designed for text-to-image generation under strong conditioning, whereas our analysis highlights intrinsic properties of diffusion dynamics, making the applicability of our framework broader and not restricted to text-to-image tasks. \textit{Besides, their exponential design is motivated by empirical intuition, while ours is supported by formal theorems, providing a stronger theoretical grounding.}  
\textbf{Energy Rectification} \cite{yang2025efficienttrainingfreehighresolutionsynthesis} studies training-free high-resolution synthesis in diffusion models and discusses the role of classifier-free guidance in generation quality. In particular, they analyze the energy decay phenomenon during high-resolution synthesis and show that tuning the classifier-free guidance hyperparameter can significantly improve results.
\textbf{Stage-wise Dynamic of CFG} \cite{jin2025stagewisedynamicsclassifierfreeguidance-mm1}  analyze CFG under multimodal conditionals and show that the sampling process can be seen as three successive stages: early \textit{direction drift}, mid  \textit{mode separation} and late \textit{constration}. And then propose a stage-wise guidance schedule. However, their dynamic guidance schedule can actually be seen as a piece-wise CFG guidance $\omega(t)$ just like \cite{kynkaanniemi2024applying}. Moreover, their method seems relies on task-specific properties such as large guidance scales $\omega(t)$,  suffering from a lack of generality.
\textbf{S$^2$-Guidance} \cite{chen2025s2guidancestochasticselfguidance-mm3} propose a novel method in which network parameters are stochastically masked to form a latent subnetwork during each forward pass, guiding the model away from potential low-quality predictions and toward higher-quality outputs. However, their approach also relies on task-specific settings, such as text-to-image and text-to-video generation. 
\cite{sadat2025trainingproblemrethinkingclassifierfree} introduce an extension of CFG called \textbf{TSG}, whose motivation is based on the structure of diffusion network. They  compute the model outputs for the clean time-step embedding and a perturbed embedding and use their difference to guide the sampling. Importantly, the motivation behind TSG arises from the interaction between conditional and timestep embeddings within the network architecture, rather than from a theoretical analysis of the diffusion framework itself.

While these approaches have shown promising improvements, they are still largely heuristic in nature and often lack rigorous theoretical justification, leaving the principles of adaptive weight design not fully understood. To address this gap, our work provides a theoretical foundation for adaptive guidance. By establishing a sequence of results (Theorems~\ref{thm mse vp}--\ref{VEthm}), we uncover structural properties of diffusion processes under different initial distributions. These insights naturally motivate the design of adaptive, theoretically grounded scaling functions. In this way, our framework offers a more robust and general basis for conditional generation.


\section{Proof of Theorems}\label{proof of thms}
In this section we give the proof of theorems: 
\subsection{Proof of Theorem \ref{thm mse vp}}\label{proof thm1}
\begin{proof}[Proof of Theorem \ref{thm mse vp}]
    For VP-SDE
    \begin{align}
    \rd x_t=-\frac12 \beta_t x_t\rd t+\sqrt{\beta_t}\rd w_t,
    \end{align}
    we can represent $x_t$ with $x_0$:
    \begin{align}\label{VP-SDE-sol}
    x_t=\alpha(t) x_0+\sigma(t)\xi_t,
    \end{align}
    where $\alpha(t) = \exp{(-\frac{1}{2}\int_0^t\beta_s\rd s)},\sigma(t)= \alpha(t)\sqrt{\int_0^t \frac{\beta_s}{\alpha^2(s)} \rd s}$, and $\xi_t\sim\mathcal N(0,I)$.

    Hence we can get the $p(x_t, t|x_0)$:
    \begin{align}
    p(x, t|x_0)=
    \frac{1}{(2\pi\sigma^2(t))^{n/2}}\exp{\Big(-\frac{||x-\alpha(t) x_0||^2}{2\sigma^2(t)}\Big)},
    \end{align}
    by using Bayes formula, we can get the probablity density function :
    \begin{align}
    p(x,t)=\int_{\mathbb R^n}
    \frac{1}{(2\pi\sigma^2(t))^{n/2}}\exp{\Big(-\frac{||x-\alpha(t) x_0||^2}{2\sigma^2(t)}\Big)}p(x_0)\rd x_0,
    \end{align}
    then we can get the score:
    \begin{align}\label{VP-SDE-score}
    \nabla\log p(x,t)&=\frac{\nabla p(x_t, t)}{p(x_t,t)}\\
    &=
    \frac{\int_{\mathbb R^n}\frac{\alpha(t) x_0-x}{\sigma^2(t)}
    \exp{\Big(-\frac{||x-\alpha(t) x_0||^2}{2\sigma^2(t)}\Big)}p(x_0)\rd x_0}{\int_{\mathbb R^n}
    \exp{\Big(-\frac{||x-\alpha(t) x_0||^2}{2\sigma^2(t)}\Big)}p(x_0)\rd x_0}\\
    &=\frac{1}{\sigma^2(t)}\Big(\alpha(t)\mathbb E[x_0|x_t=x]-x\Big).
    \end{align}

    Denote that $p(x_0|y)=\tilde p(x_0)$, consider the MSE:
    \begin{align}\label{VP-SDE-mse}
    \|\nabla\log p(x,t) - \nabla\log \tilde p(x,t)\|
    =\frac{\alpha(t)}{\sigma^2(t)}\|\mathbb E_{x_0\sim p}[x_0|x_t=x]-\mathbb E_{x'_0\sim \tilde p}[x'_0|x_t=x]\|,
    \end{align}
    then we try bounding $f(t,x) = \|\mathbb E_{x_t\sim p_t}[x_0|x_t=x]-\mathbb E_{x'_t\sim \tilde p_t}[x'_0|x'_t=x]\|$ term. Assume that $f(t,x)$ is a smooth function on $\mathbb R^n\times [0,+\infty)$, it's easy to find that
    $$f(0,x)=0, f(+\infty,x) =  \|\mathbb E_{x_0\sim p}[x_0]-\mathbb E_{x'_0\sim \tilde p}[x'_0]\|, $$ 
    hence $f(t,x)$ is a bounded function on $t$, and we denote its bound by $C(x)$. Note that we cannot say that when $t\to0$, $\frac{\alpha(t)}{\sigma^2(t)}\|\mathbb E_{x_0\sim p}[x_0|x_t=x]-\mathbb E_{x'_0\sim \tilde p}[x'_0|x_t=x]\|\to0$, because $\sigma(t)\to0$, too. 
    
    In practical engineering applications of diffusion models, the sample space is often assumed to be compact, reflecting the fact that physical quantities are naturally limited and numerical simulations are performed on finite domains. So $C(x)$ can be bounded by $C$ without loss of convince. Assume that we talk about $x_0$ on any bounded domain $K$ with $\sup_{z\in K}|z|\le R$. Let total variation distance be $\text{TV}(\mu,\nu) = \frac{1}{2}\int|\mu(\rd x)-\nu(\rd x)|$
    
       \begin{align*}
    f(t,x)
    &= \big\lVert \mathbb{E}_{x_t\sim p_t}\big[X_0\mid x_t=x\big] - \mathbb{E}_{x_t\sim \tilde p_t}\big[X_0\mid x_t=x\big]\big\rVert \\[4pt]
    &= \Big\lVert \int x_0\big(p(x_0\mid x_t=x)-\tilde p(x_0\mid x_t=x)\big)\,\mathrm{d}x_0\Big\rVert \\[4pt]
    &\le 2 M\cdot \mathrm{TV}\big(p(\cdot\mid x_t=x),\,\tilde p(\cdot\mid x_t=x)\big) \\[4pt]
    &\le 2R =C.
    \end{align*}

    Then we can rewrite \eqref{VP-SDE-mse}
    \begin{align}\label{VP-SDE-mse-final}
    \|\nabla\log p(x,t) - \nabla\log \tilde p(x,t)\|
    \leq \frac{2\alpha(t)}{\sigma^2(t)}R.
    \end{align}

\end{proof}

\subsection{Proof of Theorem \ref{thm mse ve}}\label{proof thm2}
\begin{proof}[Proof of Theorem \ref{thm mse ve}]
    For VE-SDE
    \begin{align}
    \rd x_t=\sqrt{\frac{\rd\sigma_t^2}{\rd t}}\rd w_t.
    \end{align}
    we can represent $x_t$ with $x_0$:
    \begin{align}\label{VP-SDE-sol}
    x_t=x_0+\sigma(t)\xi_t,
    \end{align}
    where $\xi_t\sim\mathcal N(0,I)$.

    Like the proof of Theorem 1, we have
    \begin{align}\label{VP-SDE-mse-final}
    \|\nabla\log p(x,t) - \nabla\log \tilde p(x,t)\|
    \leq \frac{1}{\sigma^2(t)}C.
    \end{align}

\end{proof}


\subsection{Proof of Theorem \ref{VPthm}}\label{proof thm3}
First we give Lemma \ref{thm cutoff} and Lemma \ref{thm BB-W} without proof as below:
\begin{lemma}[Cut-off Function \citep{evans1998partial}]\label{thm cutoff}
There exists a cut-off function $\eta \in C_c^\infty(B_R)$ with $0 \leq \eta \leq 1$, such that $\eta \equiv 1$ on $B_{\frac{R}{2}}$, and for any $x \in \mathbb{R}^n$,
\begin{align}
|\nabla \eta|(x) \leq \frac{C}{R} \eta^\frac{1}{2}, \quad \Delta\eta(x) \geq -\frac{C}{R^2}
\end{align}
where $C>0$ depends only on the dimension $n$.
\end{lemma}

\begin{lemma}[Bochner formula and Bakry–Émery Inequality of Heat equation with Witten Laplacian \citep{bakry2006diffusions}]\label{thm BB-W}
Define linear operator $\textnormal{L}=\Delta-\nabla\phi\cdot\nabla$, and $\nabla^2\phi$ is positive semi-definite, then for any $g\in C^3$, we have
    \begin{align}
        \frac{1}{2}\textnormal{L}|\nabla g|^2=|\nabla^2g|^2+\left<\nabla g,\nabla \textnormal{L}g\right>+\nabla g^T\nabla^2\phi\nabla g,
    \end{align}
    and furthermore
    \begin{align}
        \frac{1}{2}\textnormal{L}|\nabla g|^2\geq \frac{|\textnormal{L} g|^2}{m}+\left<\nabla g,\nabla \textnormal{L} g\right>+\nabla g^T\nabla^2\phi\nabla g
    \end{align}
    where $|\nabla^2g|^2=\Sigma_{i,j=1}^n(\partial_{ij}g)^2$ and $m\geq n$ denotes the \emph{virtual dimension}.
\end{lemma}

Then we first prove such lemma:

\begin{lemma}[Cut-off Function for Heat Equation with Witten Laplacian]\label{thm cutoff-W}
There exists a cut-off function $\eta \in C_c^\infty(B_R)$ with $0 \leq \eta \leq 1$, such that $\eta \equiv 1$ on $B_{\frac{R}{2}}$, and for any $x \in \mathbb{R}^n$, $\phi = k(|x|)x, k\geq 0$ on $B_R$, 
\begin{align}
|\nabla \eta|(x) \leq \frac{C}{R} \eta^\frac{1}{2}, \quad \Delta\eta(x) \geq -\frac{C}{R^2}, \quad \nabla\phi \cdot \nabla \eta(x) \leq 0, 
\end{align}
 $C>0$ depends only on the dimension $n$.
\end{lemma}

\begin{proof}
\textbf{Step 1. Construction of the cutoff.}
We construct a radial cutoff function by setting
\[
\eta(x) = \psi\!\left(\tfrac{|x|}{R}\right),
\]
where $\psi \in C_c^\infty([0,\infty))$ satisfies:
\[
\psi \equiv 1 \text{ on } [0,1/2], 
\qquad 
\psi \equiv 0 \text{ on } [1,\infty), 
\qquad 
\psi' \leq 0,
\]
together with the standard cutoff estimates
\[
|\psi'| \leq C \sqrt{\psi}, 
\qquad 
|\psi''| \leq C.
\]

\medskip
\noindent
\textbf{Step 2. Gradient estimate.}
Writing $r = |x|$, we compute
\[
\nabla \eta(x) = \frac{1}{R} \psi'\!\left(\tfrac{r}{R}\right)\frac{x}{r}.
\]
Hence
\[
|\nabla \eta(x)| \leq \frac{1}{R}\big|\psi'\!\left(\tfrac{r}{R}\right)\big| 
\leq \frac{C}{R}\sqrt{\eta(x)}.
\]

\medskip
\noindent
\textbf{Step 3. Laplacian estimate.}
Using the radial Laplacian formula, we have
\[
\Delta \eta(x) = \frac{1}{R^2} \psi''\!\left(\tfrac{r}{R}\right) 
+ \frac{n-1}{rR} \psi'\!\left(\tfrac{r}{R}\right).
\]
The first term is bounded by $C/R^2$ since $|\psi''| \leq C$.  
For the second term, note that $\psi' = 0$ when $r \leq R/2$, and for $r \in [R/2,R]$, we have
\[
\left|\frac{n-1}{rR}\psi'\!\left(\tfrac{r}{R}\right)\right|
\leq \frac{C}{R^2}.
\]
Therefore
\[
|\Delta \eta(x)| \leq \frac{C}{R^2}, 
\qquad 
\Delta \eta(x) \geq -\frac{C}{R^2}.
\]

\medskip
\noindent
\textbf{Step 4. Witten Laplacian estimate.}
Finally,
\[
\textnormal{L}\eta(x) = \Delta \eta(x) - kx \cdot \nabla \eta(x).
\]
Since
\[
x \cdot \nabla \eta(x) = \frac{r}{R}\psi'\!\left(\tfrac{r}{R}\right),
\]
and $\psi' \leq 0$, the term $k(x)x \cdot \nabla \eta(x) \leq 0$.  
\end{proof}

Based on this we give proof of Theorem \ref{thm Gradient Estimate-W}:
\begin{theorem}[Gradient Estimate of Heat Equation with Witten Laplacian]\label{thm Gradient Estimate-W}
   Let $u$ be a positive solution to the heat equation 
   \begin{align}\label{eq witten heat}
       \partial_tu=(\Delta-\nabla\phi\cdot\nabla)u,
   \end{align}
   on $(0,T]\times B_R.$ Assume that $\nabla^2\phi$ is positive semi-definite, $\phi = k(|x|)x, k\geq 0$ on $B_R$, then for any $(t,x)\in(0,T]\times B_{\frac R2}$, the following inequality holds:
   \begin{align}\label{}
       \frac{|\nabla u|^2}{u^2}-\alpha\frac{\partial_t u}{u}\leq\frac{m\alpha^2}{2t}+\frac{C\alpha^2}{R^2}\Big(1+\frac{\alpha^2}{\alpha - 1}\Big), 
   \end{align}
   where $m\geq n$ denotes the virtual dimension, $C(m,n)$ is a constant depends on $(m, n)$.
\end{theorem}
\begin{proof}
    We define linear operator $\textnormal{L}=\Delta-\nabla\phi\cdot\nabla$, and function $f=\log u,F=t(|\nabla f|^2-\alpha\partial_tf)$, then applying it into \eqref{eq witten heat}, we have
    \begin{align}
           &\partial_tf=\textnormal{L} f+|\nabla f|^2,\label{eq differential1}\\
           &\textnormal{L} f=-|\nabla f|^2+\partial_tf=-\frac{F}{\alpha t}-\frac{\alpha-1}{\alpha}|\nabla f|^2\label{eq differential2},\\
           &\Delta f=\textnormal{L} f+\left<\nabla\phi,\nabla f\right> = -\frac{F}{\alpha t}+\left<\nabla\phi-\frac{\alpha-1}{\alpha}\nabla f,\nabla f\right>\label{eq differential3}.
    \end{align}
    
    Based on Lemma \ref{thm BB-W} and \eqref{eq differential1},\eqref{eq differential2}, \eqref{eq differential3}, we can get
    \begin{align*}
        \textnormal{L} F &= t((\Delta-\nabla\phi\cdot\nabla)|\nabla f|^2-\alpha\partial_t((\Delta-\nabla\phi\cdot\nabla) f))\\
        &= t\left(2|\nabla^2f|^2+2\left<\nabla f,\nabla \textnormal{L}f\right>-\alpha\partial_t(\textnormal{L} f)+2\nabla f^T\nabla^2\phi\nabla f\right)\\
        &\geq t\left(\frac2m|\textnormal{L}f|^2
        +2\left<\nabla f,\nabla \textnormal{L}f\right>-\alpha\partial_t(\textnormal{L} f)+2\nabla f^T\nabla^2\phi\nabla f\right)\\
        &\geq t\frac{2|-\frac{F}{\alpha t}+\left<-\frac{\alpha-1}{\alpha}\nabla f,\nabla f\right>|^2}{m}\\
        &~~~~+t\left(2\left<\nabla f,\nabla\left(-\frac{F}{\alpha t}-\frac{\alpha-1}{\alpha}|\nabla f|^2\right)\right>-\alpha\partial_t(-\frac{F}{\alpha t}-\frac{\alpha-1}{\alpha}|\nabla f|^2)\right)\\
        &=\left(\frac{2}{m\alpha^2}\left(\frac{F^2}{t}+2(\alpha-1)F|\nabla f|^2+(\alpha-1)^2t|\nabla f|^4\right)\right)-\frac{2}{\alpha}\left<\nabla f,\nabla F\right>\\
        &~~~~-\frac{2(\alpha-1)}t{\alpha }\left<\nabla f,\nabla |\nabla f|^2\right>-\frac{F}{t}+\partial_tF+2(\alpha-1)t\left<\nabla f,\partial_t\nabla f\right>\\
        &\geq\left(\frac{2}{m\alpha^2}\left(\frac{F^2}{t}+2(\alpha-1)F|\nabla f|^2\right)\right)-\frac{2}{\alpha}\left<\nabla f,\nabla F\right>-\frac{F}{t}+\partial_tF\\
        &~~~~+\frac{2(\alpha-1)}{t}\alpha \left<\nabla f,\nabla( -|\nabla f|^2+\alpha\partial_t\nabla f)\right>\\
        &=\left(\frac{2}{m\alpha^2}\left(\frac{F^2}{t}+2(\alpha-1)F|\nabla f|^2\right)\right)-2\left<\nabla f,\nabla F\right>-\frac{F}{t}+\partial_tF, \\
    \end{align*}
    hence we have 
    \begin{align}
        (\partial_t-\textnormal{L})F\leq -\left(\frac{2}{m\alpha^2}\left(\frac{F^2}{t}+2(\alpha-1)F|\nabla f|^2\right)\right)+\frac{F}{t}+2\left<\nabla f,\nabla F\right>.
    \end{align}

    Let us consider the cut-off function $\eta$ which satisfies $\innerprod{\nabla\phi}{\nabla \eta} \geq 0$(Lemma \ref{thm cutoff-W}). We use the Bochner technique to estimate its upper bound, $\forall T'\in(0,T]$, suppose $\eta F$  attains its maximum over $(0,T']\times\bar{B}_R$ at $(t_0,x_0)$. Without loss of generality, assume $(\eta F)(t_0,x_0)>0$; otherwise, the conclusion of the theorem holds trivially. Consequently, we have $\eta(x_0),F(t_0,x_0)>0$, which implies $x_0\notin\partial B_R, t_0>0$. Thus, $(t_0,x_0)$ lies in the interior of $(B_R)_T$. Then we consider
    \begin{align*}
        (\partial_t-\textnormal{L})(\eta F)&=-F\cdot \textnormal{L}\eta-2\left<\nabla\eta,\nabla F\right>+\eta (\partial_t-\textnormal{L})F \\
        &=-F\cdot\Delta\eta + F\cdot\left<\nabla\phi, \nabla\eta\right> - 2\left<\nabla\eta,\nabla F\right>+\eta (\partial_t-\textnormal{L})F \\
        &\leq \frac{C}{R^2}F-2\left<\nabla\eta,\nabla F\right> + F\cdot\left<\nabla\phi, \nabla\eta\right>\\
        &~~~~~~+ \eta \left(-\left(\frac{2}{m\alpha^2}\left(\frac{F^2}{t}+2(\alpha-1)F|\nabla f|^2\right)\right)+\frac{F}{t}+2\left<\nabla f,\nabla F\right>\right).
    \end{align*}
    
    Applying $\nabla F=\frac{\nabla(\eta F)}{\eta}-\frac{\nabla\eta}{\eta}F$, we have
    \begin{align*}
        (\partial_t-\textnormal{L})(\eta F)(t_0,x_0)&\leq \frac{C}{R^2}F-\frac{2}{\eta}\left<\nabla\eta,\nabla (\eta F))\right>+2\frac{|\nabla\eta|^2}{\eta}F + F\cdot\left<\nabla\phi, \nabla\eta\right>
        \\&~~+\eta \left(-\left(\frac{2}{m\alpha^2}\left(\frac{F^2}{t_0}+2(\alpha-1)F|\nabla f|^2\right)\right)+\frac{F}{t_0}+2\left<\nabla f,\nabla F\right>\right)
    \end{align*}
    Using the properties of maximum$$\nabla(\eta F)(t_0,x_0)=0,\Delta(\eta F)(t_0,x_0)\leq0,\partial_t(\eta F)(t_0,x_0)=0, $$and applying Lemma \ref{thm cutoff} so that
    \begin{align}
        0&\leq \frac{C+2C^2}{R^2}F-\frac{2}{m\alpha^2}\frac{\eta F^2}{t_0}-\frac{4(\alpha-1)}{m\alpha^2}\eta F|\nabla f|^2+\frac{\eta F}{t_0}
        \\&~~~~~+2\eta \left<\nabla f,\nabla F\right> + F\cdot\left<\nabla\phi, \nabla\eta\right>\\
        &=\frac{C+2C^2}{R^2}F-\frac{2}{m\alpha^2}\frac{\eta F^2}{t_0}-\frac{4(\alpha-1)}{m\alpha^2}\eta F|\nabla f|^2+\frac{\eta F}{t_0}
        \\&~~~~~+2\left<\nabla f,\nabla (\eta F))\right>-2F\left<\nabla f,\nabla \eta\right> + F\cdot\left<\nabla\phi, \nabla\eta\right>\\
        &=\frac{C+2C^2}{R^2}F-\frac{2}{m\alpha^2}\frac{\eta F^2}{t_0}-\frac{4(\alpha-1)}{m\alpha^2}\eta F|\nabla f|^2+\frac{\eta F}{t_0}
        \\&~~~~~-2F\left<\nabla f,\nabla \eta\right> + F\cdot\left<\nabla\phi, \nabla\eta\right>\label{ineq diff1}, 
    \end{align}
    then let us consider two of the terms $\frac{4(\alpha-1)}{m\alpha^2}\eta F|\nabla f|^2+2F\left<\nabla f,\nabla \eta\right>$,
    \begin{align*}
        \frac{4(\alpha-1)}{m\alpha^2}\eta F|\nabla f|^2+2F\left<\nabla f,\nabla \eta\right>&\geq \frac{4(\alpha-1)}{m\alpha^2}\eta F|\nabla f|^2-2F|\nabla f||\nabla \eta|\\
        &\geq \frac{4(\alpha-1)}{m\alpha^2}\eta F|\nabla f|^2\\&~~-F(\frac{4(\alpha-1)R^2}{m\alpha^2C^2}|\nabla f|^2|\nabla\eta|^2+\frac{m\alpha^2C^2}{4(\alpha-1)R^2})\\
        &\geq -\frac{m\alpha^2C^2}{4(\alpha-1)R^2}F
    \end{align*}
    then inequality \ref{ineq diff1} can be turn into
    \begin{align*}
        0
        &\leq \left(\frac{m\alpha^2C^2}{4(\alpha-1)R^2}+\frac{C+2C^2}{R^2}\right)F-\frac{2}{m\alpha^2}\frac{\eta F^2}{t_0}+\frac{\eta F}{t_0} + F\cdot\left<\nabla\phi, \nabla\eta\right>,
    \end{align*}
     
    then we divide $F$ and then get 
    \begin{align*}
        \eta F(t_0,x_0)&\leq  \frac{m\alpha^2}{2}t_0 \left(\frac{m\alpha^2C^2}{4(\alpha-1)R^2}+\frac{C+2C^2}{R^2}+\frac{\eta}{t_0} + \left<\nabla\phi, \nabla\eta\right>\right)\\
        &\leq \frac{m\alpha^2}{2}t_0\left( \frac{m\alpha^2C^2}{4(\alpha-1)R^2}+\frac{C+2C^2}{R^2}+\frac{1}{t_0}\right)\\
        &\leq\frac{m\alpha^2}{2}+\frac{m\alpha^2}{2}\left(\frac{m\alpha^2C^2}{4(\alpha-1)R^2}+\frac{C+2C^2}{R^2}\right)t_0\\
        &\leq \frac{m\alpha^2}{2}+\frac{C_1\alpha^2}{R^2}\left(\frac{\alpha^2}{\alpha-1}+1\right)T', \\
        &(C_1 = \max\{m^2C^2/8, C^2+C/2\}),
    \end{align*}
    On $B_{\frac{R}{2}}$, $\eta = 1, \nabla\eta=0$, so for all $(t,x)\in (0,T']\times B_{\frac{R}{2}}$
    \begin{align*}
       t(|\nabla f|^2-\alpha\partial_tf)|_{t=T'}&= F(T',x) = \eta F(T',x) \leq \eta F(t_0,x_0)\\
       &\leq   \frac{m\alpha^2}{2}+\frac{C_1\alpha^2}{R^2}\Big(\frac{\alpha^2}{\alpha-1}+1\Big)T',
    \end{align*}
    $T'$ is arbitrary, so
    \begin{align}
        (|\nabla f|^2-\alpha\partial_tf)&\leq   \frac{m\alpha^2}{2t}+\frac{C_1\alpha^2}{R^2}\Big(\frac{\alpha^2}{\alpha-1}+1\Big).
    \end{align}
\end{proof}

From Theorem \ref{thm Gradient Estimate-W} we can conclude Theorem \ref{thm Witten heat inq}
\begin{theorem}[Harnack-type Inequality of Heat Equation with Witten Laplacian ]\label{thm Witten heat inq}
Let $u$ be a positive solution of the heat equation $\partial_t u = \textnormal{L} u$ in $ (0,T] \times B_R$, where $\alpha > 1$.  
For any $x_1, x_2 \in B_{\frac{R}{2}}$ and $0 < t_1 < t_2 \leq T$, the following inequality holds:
\begin{align}
    u(x_1,t_1) \leq u(x_2,t_2) \left(\frac{t_2}{t_1} \right)^{\frac{m\alpha}{2}}
    \exp{\left( \frac{\alpha^2 \|x_1-x_2\|^2}{4(t_2 - t_1)}
    + \frac{C\alpha}{R^2} \left( 1 + \frac{\alpha^2}{\alpha - 1} \right)(t_2 - t_1) \right)},
\end{align}
where $C = C(m,n)$.
\end{theorem}
\begin{proof}
Let $f=\log u$. Consider the line segment
\[
    L(s) = (1-s)(t_2, x_2) + s(t_1, x_1).
\]
We have
\begin{align*}
    \log\frac{u(x_1,t_1)}{u(x_2,t_2)}
    &= \int_0^1 \frac{d}{ds} f(L(s)) \, ds \\
    &= \int_0^1 \left[ \nabla f(L(s)) \cdot (x_1 - x_2)
    + \partial_t f(L(s)) (t_1 - t_2) \right] ds.
\end{align*}
Moreover, using the inequality
\[
    -\partial_t f \leq -\frac{1}{\alpha} |\nabla f|^2 + \frac{m\alpha}{2t}
    + \left[\frac{C\alpha}{R^2} \left( \frac{\alpha^2}{\alpha - 1} + 1 \right)\right],
\]
we get
\begin{align*}
    \log\frac{u(x_1,t_1)}{u(x_2,t_2)}
    &\leq \int_0^1 \Big[ |\nabla f(L(s))|\, |x_1 - x_2| \\
    &\quad + \Big( -\frac{1}{\alpha} |\nabla f|^2(L(s))
    + \frac{m\alpha}{2[(1-s) t_2 + s t_1]}\\
    &\quad + \frac{C\alpha}{R^2} \left( \frac{\alpha^2}{\alpha - 1} + 1 \right)\Big) (t_2 - t_1) \Big] ds.
\end{align*}
Using the inequality
\[
    |\nabla f(L(s))|\, |x_1 - x_2|
    - \frac{t_2 - t_1}{\alpha} |\nabla f|^2(L(s))
    \leq \frac{\alpha\, d^2(x_1,x_2)}{4(t_2 - t_1)},
\]
we obtain
\begin{align*}
    \log\frac{u(x_1,t_1)}{u(x_2,t_2)}
    &\leq \frac{\alpha\, d^2(x_1,x_2)}{4(t_2 - t_1)}
    + \frac{m\alpha}{2} \ln\frac{t_2}{t_1}
    + \frac{C\alpha}{R^2} \left( \frac{\alpha^2}{\alpha - 1} + 1 \right) (t_2 - t_1).
\end{align*}
\end{proof}

Finally, we can prove Theorem \ref{VPthm}:
\begin{proof}[Proof of Theorem \ref{VPthm}]
        The VP-SDE is given by
    \begin{equation}
        \rd x_t = -\frac{1}{2}\beta_t x_t \rd t + \sqrt{\beta_t} \rd W_t,
    \end{equation}
    and its corresponding Fokker-Planck equation (FPE) is
    \begin{equation}
        \frac{\partial p_t(x)}{\partial t} = \frac{1}{2}\beta_t \left( \nabla_x \cdot [x p_t(x)] + \Delta_x p_t(x) \right).
    \end{equation}
    
    We can reparameterize $t$ by letting $ds = \frac{1}{2}\beta_t dt$. Then,
    \begin{align}
        s(t) &= \frac{1}{2}\int_0^t \beta_r dr, \\
        \frac{\rd}{\rd t} &= \frac{1}{2}\beta_t \frac{\rd}{\rd s}.
    \end{align}
    Thus,
    \begin{equation}
        \frac{\partial p_{t(s)}(x)}{\partial s} = \frac{\partial p_t}{\partial t} \frac{\rd t}{\rd s} = \frac{\partial p_t(x)}{\partial t} \cdot \frac{1}{\frac{1}{2}\beta_t} = \nabla_x \cdot [x p_t(x)] + \Delta_x p_t(x).
    \end{equation}
    
    For this new FPE
    \begin{equation}
        \frac{\partial p_{t(s)}(x)}{\partial s} = \nabla_x \cdot [x p_{t(s)}(x)] + \Delta_x p_{t(s)}(x),
    \end{equation}
    the corresponding SDE is
    \begin{equation}
        \rd x_{t(s)} = -x_{t(s)} \rd s + \sqrt{2} \rd W_s.
    \end{equation}
    
    Assume $p(x,t)$ is a positive solution to this FPE, and let $u(x,t) = p(x,t) e^{|x|^2/2}$. Computing the right-hand side:
    \begin{align}
        \nabla(xp) &= x\nabla p + n p = (n u + x \nabla u - |x|^2 u) e^{-|x|^2/2}, \\
        \Delta p &= \nabla \cdot [(\nabla u - x u) e^{-x^2/2}] = [\Delta u - n u - 2x \nabla u + |x|^2 u] e^{-|x|^2/2}, \\
        \nabla(xp) + \Delta p &= [\Delta u - x \nabla u] e^{-|x|^2/2}.
    \end{align}
    
    Thus, the FPE for $u$ is
    \begin{equation}
        \frac{\partial u_{t(s)}(x)}{\partial s} = \Delta u - x \cdot \nabla u = \Delta u - \nabla \phi \cdot \nabla u, \quad \phi = \frac{|x|^2}{2},
    \end{equation}
    which satisfies the equation in \textbf{Theorem \ref{thm Witten heat inq}}, and we can easily figure out that $k(|x|)=1>0$.
    
    Therefore, for any $ \alpha>1, x_1,x_2\in M, 0<s_1<s_2<+\infty$, and let $R\to\infty$, the following inequality holds:
    \begin{align}
    u(x_1, t(s_1)) \leq u(x_2, t(s_2)) \left(\frac{s_2}{s_1} \right)^{\frac{m\alpha}{2}} \exp\left(\frac{\alpha^2 \|x_1-x_2\|^2}{4(s_2 - s_1)}\right).
    \end{align}
    
Rewriting it in terms of $p$, we obtain
    \begin{align}
    p(x_1, t(s_1)) \leq p(x_2, t(s_2)) \left(\frac{s_2}{s_1} \right)^{\frac{m\alpha}{2}} \exp\left(\frac{\alpha^2 \|x_1-x_2\|^2}{4(s_2 - s_1)} + \frac{\|x_2\|^2 - \|x_1\|^2}{2}\right).
    \end{align}
\end{proof}

\subsection{Proof of Theorem \ref{VEthm}}\label{proof thm4}
First we give Lemma \ref{thm BB} without proof as below:
\vspace{2mm}
\begin{lemma}[Bochner Formula and Bakry–Émery Inequality \citep{bakry2006diffusions}]\label{thm BB}
For any $g\in C^3$, we have
    \begin{align}
        \frac{1}{2}\Delta|\nabla g|^2=|\nabla^2g|^2+\left<\nabla g,\nabla \Delta g\right>,
    \end{align}
    and furthermore
    \begin{align}
        \frac{1}{2} \Delta |\nabla g|^2\geq \frac{|\Delta g|^2}{n}+\left<\nabla g,\nabla \Delta g\right>
    \end{align}
    where $|\nabla^2g|^2=\Sigma_{i,j=1}^n(\partial_{ij}g)^2$.
\end{lemma}

Based on this we give proof of Theorem \ref{thm Gradient Estimate}:
\begin{theorem}[Gradient Estimate of Heat equation]\label{thm Gradient Estimate}
   Let $u$ be a positive solution to the heat equation 
   \begin{align}\label{eq heat}
       \partial_tu=\Delta u,
   \end{align}
   on $(0,T]\times B_R.$ Then for any $(t,x)\in(0,T]\times B_{\frac R2}$, the following inequality holds:
   \begin{align}\label{}
       \frac{|\nabla u|^2}{u^2}-\alpha\frac{\partial_t u}{u}\leq\frac{n\alpha^2}{2t}+\frac{C\alpha^2}{R^2}\Big(1+\frac{\alpha^2}{\alpha - 1}\Big), 
   \end{align}
   where $C(n)$ is a constant depends on $n$.
\end{theorem}
\begin{proof}
   Like the proof of Theorem \ref{thm Gradient Estimate-W}, just turn $\textnormal{L}$ into $\Delta$ and then we can get the conclusion.
\end{proof}

From Theorem \ref{thm Gradient Estimate} we can conclude Theorem \ref{thm heat inq}: 

\begin{theorem}[Harnack-type Inequality of Heat Equation]\label{thm heat inq}
Let $u$ be a positive solution of the heat equation $\partial_t u = \Delta u$ in $ (0,T] \times B_R$, where $\alpha > 1$.  
For any $x_1, x_2 \in B_{\frac{R}{2}}$ and $0 < t_1 < t_2 \leq T$, the following inequality holds:
\begin{align}
    u(x_1,t_1) \leq u(x_2,t_2) \left(\frac{t_2}{t_1} \right)^{\frac{m\alpha}{2}}
    \exp{\left( \frac{\alpha^2 \|x_1-x_2\|^2}{4(t_2 - t_1)}
    + \frac{C\alpha}{R^2} \left( 1 + \frac{\alpha^2}{\alpha - 1} \right)(t_2 - t_1) \right)},
\end{align}
where $C = C(n)$.
\end{theorem}
\begin{proof}
Like the proof of \ref{thm Witten heat inq}.
\end{proof}

Finally, we can prove Theorem \ref{VEthm}:
\begin{proof}[Proof of Theorem \ref{VEthm}]
        The VE-SDE form is given by $\rd x_t=\sqrt{\frac{\rd\sigma_t^2}{dt}}\rd W_t$, and its corresponding FPE form is
$$ \frac{\partial p_t(x)}{\partial t}=\frac{1}{2}\frac{\rd \sigma_t^2}{dt}\Delta_x(p_t(x)).$$
We can reparameterize  $t$ by letting $s=\frac12\sigma_t^2$, which gives $\frac{\rd s}{\rd t}=\frac12\frac{\rd\sigma^2}{\rd t}$. Therefore,
$$\frac{\partial p_{t(s)}(x)}{\partial s}=\frac{\partial p_t}{\partial t}\frac{\rd t}{\rd s}=\frac{\partial p_t(x)}{\partial t}\frac{1}{\frac12\frac{\rd\sigma^2}{\rd t}}=\Delta_x(p_t(x)).$$

For this new FPE $\frac{\partial p_{t(s)}(x)}{\partial s}=\Delta_x(p_t(x))$, its corresponding SDE form is:
$$\rd x_{t(s)}=\sqrt{2}\rd W_s.$$
Assume $p(x,t)$ is the fundamental solution of this FPE, satisfying Theorem \ref{thm heat inq}.

Thus, for any $\alpha>1$, $x_1,x_2\in M$, and $0<s_1<s_2<+\infty$, let $R\to\infty$, the following inequality holds:
    \begin{align}
        u(x_1,t(s_1))) \leq u(x_2,t(s_2)) \left(\frac{s_2}{s_1} \right)^{\frac{n\alpha}{2}}\exp{\left(\frac{\alpha^2\|x_1-x_2\|^2}{4(s_2-s_1)}\right)}.
    \end{align}
\end{proof}

\section{Relationship Between MSE Bound and Harnack-type Inequality}\label{relationship}
In this section we provide a deeper insight into the connection between Theorems \ref{thm mse vp} and \ref{VPthm}: they respectively lead to Theorems \ref{KL-bound}  and \ref{thm ent-cost}. In essence, these two results offer complementary perspectives on the evolution of the KL divergence.
\subsection{Harnack-type inequality to KL-divergence}
Starting from Harnack-type inequality, we can arrive at log-Harnack inequality. Consider SDE 
$$\rd X_t = - X_t \rd t + \sqrt 2 \rd W_t, $$
we obtain Theorem \ref{thm log-harnack}:
\begin{theorem}[log-Harnack inequality]\label{thm log-harnack}
    Let $u(t,x)=P_t f(x)=\int \varphi_t(x,y)f(y)\rd y$ with the OU Mehler kernel
    $\varphi_t(x,y)=(2\pi s_t)^{-n/2}\exp\!\big(-\frac{|y-e^{-t}x|^2}{2s_t}\big)$, $s_t=1-e^{-2t}$.
    Assume $\operatorname{supp}(f)\subset B(0,R)$. 
    Then for every $t>0$ and every $x,y\in\mathbb R^n$,
    \begin{align*}
        P_t\log f(y) \leq  \log P_tf(x)+ \mid x-y\mid \sup_{z\in[x,y]} \sqrt{\frac{m\alpha^2}{2t}+\alpha\left(\frac{e^{-t}}{s_t}S'(z,t)\right)},
    \end{align*}
    where $S'(x, t)=(( R^2 + e^{-2t} \mid x\mid^2 + 2e^{-t}R\mid x\mid )^2+ \mid x \mid R + e^{-t} \mid x\mid ^2 - ne^{-t}), [x,y]:= \{x+\theta(y-x):\theta\in[0,1]\}$.
    In particular, on any bounded domain $K$ with $\sup_{z\in K}|z|\le M$ one has
    \begin{align}\label{LH-K}
        P_t\log f(y) &\leq \log P_tf(x)+ \mid x-y\mid \sqrt{\frac{m\alpha^2}{2t}+\alpha\left(\frac{e^{-t}}{s_t}S'(\mid x\mid=M,t)\right)}\\
        &=\log P_tf(x) + S_K(t)\mid x-y\mid.
    \end{align}
\end{theorem}
\begin{proof}
     From Theorem \ref{thm Gradient Estimate-W} we conclude that a Gradient estimate holds on $\mathbb R^n$:
    \[
    \frac{|\nabla u|^2}{u^2}-\alpha\frac{\partial_t u}{u}\le \frac{m\alpha^2}{2t},
    \]
    where $\alpha>1, m>n$. For $\varphi_t$, we have
    \begin{align}
        \nabla_x \log \varphi_t(x,y) = \frac{e^{-t}}{s_t}(y-e^{-t}x),\\
        \Delta_x\log \varphi_t(x,y) = -\frac{ne^{-2t}}{s_t}.
    \end{align}
    Thus 
    \begin{align}
    \partial_t \log u &= \frac{\text{L}_x u}{u} \\
    &= \frac{\int (\Delta_x \varphi_t(x,y) - x \cdot \nabla_x \varphi_t(x,y)) f(y)  \mathrm{d} y}{\int \varphi_t(x,y) f(y)  \mathrm{d} y} \\
    &= \frac{\int (\Delta_x \log \varphi_t(x,y) + \|\nabla_x \log \varphi_t(x,y)\|^2 - x \cdot \nabla_x \log \varphi_t(x,y)) \varphi_t(x,y) f(y)  \mathrm{d} y}{\int \varphi_t(x,y) f(y)  \mathrm{d} y} \\
    &= \mathbb{E}_{Y \sim \pi_{t,x}} \left[ \Delta_x \log \varphi_t(x,Y) + \|\nabla_x \log \varphi_t(x,Y)\|^2 - x \cdot \nabla_x \log \varphi_t(x,Y) \right], \\
    &= \mathbb{E}_{Y \sim \pi_{t,x}} \left[ -\frac{n e^{-2t}}{s_t} + \frac{e^{-2t}}{s^2_t} \| Y - e^{-t} x \|^2 - \frac{e^{-t}}{s_t} (x \cdot Y - e^{-t} x^2) \right],
\end{align}
    where $\pi_{t,x}=\frac{\varphi_t(x,y)f(y)}{\int \varphi_t(x,y)f(y)\rd y} $. As $\operatorname{supp}(f)\subset B(0,R)$, 
    \begin{align*}
    \partial_t \log u 
    &= \mathbb{E}_{Y \sim \pi_{t,x}} \left[ -\frac{n e^{-2t}}{s_t} + \frac{e^{-2t}}{s^2_t} \| Y - e^{-t} x \|^2 - \frac{e^{-t}}{s_t} (x \cdot Y - e^{-t} \mid x\mid^2) \right]\\
    &\leq -\frac{n e^{-2t}}{s_t} + \frac{e^{-2t}}{s^2_t} ( R^2 + e^{-2t} \mid x\mid^2 + 2e^{-t}R\mid x\mid )^2 \\
    &~~~~~~+ \frac{e^{-t}}{s_t} (\mid x \mid R + e^{-t} \mid x\mid ^2)\\
    &\leq \frac{e^{-t}}{s_t}(( R^2 + e^{-2t} \mid x\mid^2 + 2e^{-t}R\mid x\mid )^2+ \mid x \mid R + e^{-t} \mid x\mid ^2 - ne^{-t})\\
    &=\frac{e^{-t}}{s_t} S'(x,t).
    \end{align*}

    Thus 
    \begin{align}
        \|\nabla \log u\|^2\le \frac{m\alpha^2}{2t}+\alpha\left(\frac{e^{-t}}{s_t}S'_K(x,t)\right),\\
        \|\nabla \log u\|\le \sqrt{\frac{m\alpha^2}{2t}+\alpha\left(\frac{e^{-t}}{s_t}S'(x,t)\right)},
    \end{align}
    we can easily get that 
    \begin{align}
        \log u(t,y)-\log u(t,x)\leq \mid x-y\mid \sup_{z\in[x,y]} \sqrt{\frac{m\alpha^2}{2t}+\alpha\left(\frac{e^{-t}}{s_t}S'(z,t)\right)},
    \end{align}

    by Jesen's inequality, we have 
    \begin{align}
        P_t\log f(y) &\leq  \log P_tf(x)+ \mid x-y\mid \sup_{z\in[x,y]} \sqrt{\frac{m\alpha^2}{2t}
        +\alpha\left(\frac{e^{-2t}}{s_t}S'(z,t)\right)},
    \end{align}
    as desired.
\end{proof}

Thus we obtain theorem below.
\begin{theorem}[entropy–cost inequality]\label{thm ent-cost}
   Let $K\subset\mathbb R^n$ be compact, and assume the transition kernels $P_t(x,\cdot)=\varphi_t(x,\cdot)\rd y$ satisfy the pointwise log-Harnack inequality \ref{LH-K} above for all $x,y\in K$. Then for any two probability measures $\mu,\nu$ supported in $K$ and any coupling $\pi\in\Pi(\mu,\nu)$,
\[
\mathrm{KL}( P_t\nu\;\|\; P_t\mu)
\le \iint |x-y|\,S_K(t)\,\pi(dx,dy)
= S_K(t)\,\mathbb E_{\pi}[|X-Y|].
\]
Taking the infimum over couplings,
\begin{align}
    \mathrm{KL}( P_t\nu\;\|\; P_t\mu)\le S_K(t)\,W_1(\mu,\nu)\le S_K(t)\,W_2(\mu,\nu),
\end{align}
so in particular the KL at time \(t\) is bounded by a compact-set constant \(S_K(t)\) times the initial Wasserstein distance.
\end{theorem}
\begin{proof}
Recall the variational (Donsker--Varadhan) formula for relative entropy of two probability densities \(\rho,\mu\) \citep{https://doi.org/10.1002/cpa.3160280102}:
\[
\mathrm{KL}( P_t\nu\;\|\; P_t\mu) 
= \sup_{\phi\in B_b}\Big\{\int \phi(z)\,P_t\nu (dz) - \log \int e^{\phi(z)}\,P_t\mu (dz)\Big\},
\]
where \(B_b\) denotes bounded measurable functions, $ P_t\nu(dz) = \int_y \varphi_t(y,z)\,\nu(dy)\,dz, 
~~
P_t\mu(dz) = \int_x \varphi_t(x,z)\,\mu(dx)\,dz.$

For an arbitrary bounded \(\phi\) set \(f=e^\phi\ge1\). Then
\[
\int \phi(z)\,\varphi_t(y,z)\,dz \;\le\; 
\log \int e^{\phi(z)}\,\varphi_t(x,z)\,dz + |x-y|\,S_K(t).
\]

Taking the supremum over all bounded \(\phi\) yields exactly 
\[
\mathrm{KL}\!\left(\varphi_t(y,\cdot)\,\Big\|\,\varphi_t(x,\cdot)\right)
\;\le\; |x-y|\,S_K(t).
\]
Now fix any coupling \(\pi\in\Pi(\mu,\nu)\). By convexity of KL under mixtures (or the standard coupling inequality),
\begin{align*}
    \mathrm{KL}( P_t\nu\;\|\; P_t\mu)
&= \mathrm{KL}\Big(\int \varphi_t(y,\cdot)\,\nu(dy)\;\Big\|\;\int \varphi_t(x,\cdot)\,\mu(dx)\Big)\\
& \le \iint \mathrm{KL}\!\left(\varphi_t(y,\cdot)\,\|\,\varphi_t(x,\cdot)\right)\,\pi(dx,dy).
\end{align*}

Using the kernel bound and factoring \(S_K(t)\) yields
\[
\mathrm{KL}( P_t\nu\;\|\; P_t\mu)\le \iint |x-y|S_K(t)\,\pi(dx,dy) = S_K(t)\mathbb E_\pi[|X-Y|].
\]
Taking infimum over \(\pi\) gives the \(W_1\) form. Finally the monotonicity \(W_1\le W_2\) yields the stated \(W_2\)-bound.
\end{proof}

\subsection{Score MSE bound to KL-divergence}
\begin{definition}[Relative Fisher Information]
Let \(\nu\) and \(\mu\) be two probability measures on \(\mathbb{R}^n\) such that \(\nu\) is absolutely continuous with respect to \(\mu\). The \emph{relative Fisher information} of \(\nu\) with respect to \(\mu\) is defined by
\[
I(\nu \,\|\, \mu) := \int_{\mathbb{R}^n} \Big\| \nabla \log \frac{d\nu}{d\mu}(x) \Big\|^2 \, d\nu(x),
\]
where \(\frac{d\nu}{d\mu}\) denotes the Radon--Nikodym derivative of \(\nu\) with respect to \(\mu\), and \(\nabla \log \frac{d\nu}{d\mu}\) is the \emph{score function} of \(\nu\) relative to \(\mu\).

Intuitively, \(I(\nu \,\|\, \mu)\) measures the squared \(L^2(\nu)\)-distance between the score functions of \(\nu\) and \(\mu\).
\end{definition}

\begin{theorem}\label{entropy eq}
    Let $X_t\in \mathbb R^n$ be the output of the SDE 
    \begin{align}\label{SDE_IT}
        \rd X_t = a(X_t,t)\rd t+g(t)\rd W_t.
    \end{align}
    
    Then for the above KL-divergence, we have 
    \begin{align}
        \frac{\rd}{\rd t}\mathrm{KL}( P_t\nu\;\|\; P_t\mu) &= -\frac{1}{2}g^2(t) I( P_t\nu\;\|\; P_t\mu).
    \end{align}

    For OU process $$\rd X_t = - X_t \rd t + \sqrt 2 W_t, $$ we have the form:
    \begin{align}
        \frac{\rd}{\rd t}\mathrm{KL}( P_t\nu\;\|\; P_t\mu) &= -I( P_t\nu\;\|\; P_t\mu).
    \end{align}
\end{theorem}
\begin{proof}
    We note $P_t\nu =p_t, P_t\mu =q_t$ for convenience.
    FPE of \eqref{SDE_IT} is 
    \begin{align*}
        \partial_tp_t = -\nabla\cdot(ap_t)+\frac{1}{2}\sigma^2(t)\Delta p_t,
    \end{align*}
    as for differential entropy $H(X_t)=-\int p_t\log p_t\rd x$, we obtain
    \begin{align*}
        \frac{\rd}{\rd t}H(X_t) &= -\int \partial_t p_t \log p_t\rd x - \int \partial_t p_t \rd x\\
        &= -\int \partial_t p_t \log p_t\rd x\\
        &= \int \nabla\cdot(ap_t) \log p_t\rd x - \int \frac{1}{2}g^2(t)\Delta p_t \log p_t\rd x,
    \end{align*}

    then we calculate the terms in the above equation,
    \begin{align*}
        \int \nabla\cdot(ap_t) \log p_t\rd x &=
        (-1)\int\innerprod{a_tp_t}{\frac{\nabla p_t}{p_t}}\\
        &=-\int \innerprod{a_t}{\nabla p_t}\\
        &= \mathbb E_{p_t}[\nabla \cdot a_t],
    \end{align*}
    using $\Delta \log p = \Delta p/p-(\nabla \log p)^2$,
    \begin{align*}
        \int \frac{1}{2}g^2(t)\Delta p_t \log p_t\rd x &= \frac{1}{2}g^2(t)
        \int p_t\Delta\log p_t\\
        &=\frac{1}{2}g^2(t)
        \int p_t(\Delta p_t/p_t-(\nabla \log p_t)^2)\\
        &= -\frac{1}{2}g^2(t) \int p_t(\nabla \log p_t)^2,
    \end{align*}
    so we obtain 
    \begin{align*}
        \frac{\rd}{\rd t}H(X_t) &= \frac{1}{2}g^2(t) \int p_t(\nabla \log p_t)^2+\mathbb E_{p_t}[\nabla \cdot a_t].
    \end{align*}

    Then we consider the term $S(p_t,q_t)=-\int p_t\log q_t\rd x$,
    \begin{align*}
        \frac{\rd}{\rd t}S(p_t,q_t) &= -\int \partial_t p_t \log q_t\rd x - \int \frac{p_t}{q_t}\partial_t q_t \rd x\\
        &= \int \nabla\cdot(ap_t) \log q_t\rd x - \int \frac{1}{2}g^2(t)\Delta p_t \log q_t\rd x \\
        & ~~~ +\int \nabla\cdot(aq_t) \frac{p_t}{q_t} \rd x - \int \frac{1}{2}g^2(t)\Delta q_t \frac{p_t}{q_t}\rd x,
    \end{align*}
    then we calculate the terms in the above equation,
    \begin{align*}
        \int \nabla\cdot(ap_t) \log q_t\rd x &=
        (-1)\int\innerprod{a_tp_t}{\frac{\nabla q_t}{q_t}}\\
        &=-\int\innerprod{a_t}{\nabla \log q_t}p_t,
    \end{align*}
    \begin{align*}
        \int \nabla\cdot(aq_t) \frac{p_t}{q_t} \rd x &=
        (-1)\int\innerprod{a_tq_t}{\frac{q_t\nabla p_t-p_t\nabla q_t}{q^2_t}}\\
        &=-\int\innerprod{a_t}{\nabla p_t-\frac{p_t\nabla q_t}{q_t}}\\
        &=\mathbb E_{p_t}[\nabla \cdot a_t] +\int\innerprod{a_t}{\nabla \log q_t}p_t,
    \end{align*}
    \begin{align*}
        \int \frac{1}{2}g^2(t)\Delta q_t \frac{p_t}{q_t}\rd x &= -\frac{1}{2}g^2(t)
        \int \innerprod{\nabla q_t}{\frac{q_t\nabla p_t-p_t\nabla q_t}{q^2_t}}\\
        &=-\frac{1}{2}g^2(t)
        \int \innerprod{\nabla \log q_t}{\nabla \log p_t - \nabla \log q_t}p_t,
    \end{align*}
    \begin{align*}
        \int \frac{1}{2}g^2(t)\Delta p_t \log q_t\rd x &= -\frac{1}{2}g^2(t)
        \int \int \innerprod{\nabla \log p_t}{\nabla \log q_t}p_t,
    \end{align*}
    
    so we obtain 
    \begin{align*}
        \frac{\rd}{\rd t}S(p_t,q_t) &= -\frac{1}{2}g^2(t) \int p_t[(\nabla \log q_t)^2 - 2\innerprod{\nabla \log p_t}{\nabla \log q_t}]+\mathbb E_{p_t}[\nabla \cdot a_t].
    \end{align*}

    Then we have
    \begin{align*}
        \frac{\rd}{\rd t}\mathrm{KL}(p_t\|q_t) &= -\frac{1}{2}g^2(t) I(p_t\,\|\, q_t) .
    \end{align*} 
\end{proof}

Still, we consider SDE 
$$\rd X_t = - X_t \rd t + \sqrt 2 W_t, $$
then we obtain conclusion below via Theorem \ref{entropy eq} and \ref{thm mse vp}:
\begin{theorem}[KL Bound for Ornstein--Uhlenbeck SDE]\label{KL-bound}
Consider the Ornstein--Uhlenbeck SDE
\[
\mathrm{d}X_t = - X_t \, \mathrm{d}t + \sqrt{2} \, \mathrm{d}W_t,
\]
by Theorem \ref{thm mse vp} we obtain
\[
I(p_t \,\|\, q_t) \leq 4R^2\frac{e^{-2t}}{(1 - e^{-2t})^2},
\]
and let \(p_t\) and \(q_t\) be the distributions of two solutions with different initial conditions. Then, there exists a constant \(C>0\) such that for all \(t \ge 0\),
\[
\mathrm{KL}(p_t \,\|\, q_t) 
= \int_t^\infty I(p_s \,\|\, q_s) \, \mathrm{d}s
\le \int_t^\infty   4R^2\frac{e^{-2s}}{(1 - e^{-2s})^2} \, \mathrm{d}s
\le 2R^2 \frac{e^{-2t}}{1-e^{-2t}},
\]
where \(I(p_s \,\|\, q_s)\) denotes the relative Fisher information (or score MSE) of \(p_s\) with respect to \(q_s\).

In particular, this provides an explicit upper bound for the KL divergence between \(p_t\) and \(q_t\) in terms of \(t\).
\end{theorem}

\subsection{Conclusion}
Via Theorem \ref{thm mse vp} and \ref{VPthm}, we can get Theorem \ref{KL-bound} and \ref{thm ent-cost}, which both bound the KL-divergence $\mathrm{KL}(p_t \,\|\, q_t)$.  

We can observe that these two approaches are closely related in spirit: 

\begin{itemize}
    \item The MSE-bound approach (Theorem \ref{KL-bound}) directly controls the relative Fisher information 
    \[
    I(p_t \,\|\, q_t) = \mathbb E_{p_t}[\,|s_{p_t}-s_{q_t}|^2\,],
    \] 
    and then integrates it over time to obtain an explicit upper bound for the KL-divergence. 

    \item The Harnack inequality approach (Theorem \ref{thm ent-cost}) instead provides a pointwise control on the semigroup, which, via coupling and Wasserstein distances, leads to a KL upper bound of the form
    \[
    \mathrm{KL}( P_t\nu\;\|\; P_t\mu) \le S_K(t)\, W_1(\mu,\nu) \le S_K(t)\, W_2(\mu,\nu).
    \]

    \item In essence, both methods link the KL divergence at time $t$ to some notion of discrepancy at the initial time: MSE-bound does it via the score difference (relative Fisher information), while Harnack-bound does it via transport distances ($W_1$ or $W_2$). The MSE bound can be seen as a “local-in-space” version of the Harnack control: if the pointwise kernel control from Harnack implies a bound on $\nabla \log p_t$, then integrating it yields a Fisher-information-type bound. Thus, the two approaches are complementary perspectives on how initial differences propagate under the dynamics of the SDE.
\end{itemize}

This observation highlights that controlling either the score differences or the pointwise semigroup can provide rigorous quantitative bounds on the evolution of KL divergence in diffusion processes.

\section{Algorithm Comparison}\label{sec alg}
Figure~\ref{algorithm} compares standard CFG and C$^2$FG. At each timestep $t$ during generation, the C$^2$FG update replaces the standard CFG as follows:
\[
\epscfgnc(\x_t) = \epsnullnc(\x_t) + \omega(t) \big[\epscond(\x_t) - \epsnullnc(\x_t)\big].
\]
\begin{figure*}[!ht]
\begin{minipage}{.48\textwidth}
    \begin{algorithm}[H]
    \small
           \caption{Reverse Diffusion with CFG}
           \label{alg:cfg}
            \begin{algorithmic}[1]
             \Require $\x_T \sim \Nc(0, \rmI_d), 0 \leq \textcolor{cfg}{\omega} \in \mathbb{R}$
             \For{$i=T$ {\bfseries to} $1$}
                
                 \State{$\epscfgnc(\x_t) = \epsnullnc(\x_t) + \textcolor{cfg}{\omega} [\epscond(\x_t) - \epsnullnc(\x_t)]$}
                 \State{$\tweediecfgnc(\x_t) \gets (\x_t - \sqrt{1-\bar\alpha_t}\epscfgnc(\x_t))/\sqrt{\bar\alpha_t}$}
                 \State{$\x_{t-1} = \sqrt{\bar\alpha_{t-1}} \tweediecfgnc(\x_t) + \sqrt{1-\bar\alpha_{t-1}}\epscfgnc(\x_t)$}
              \EndFor
              \State {\bfseries return} $\x_0$
            \end{algorithmic}
    \end{algorithm}
\end{minipage}
\begin{minipage}{.48\textwidth}
    \begin{algorithm}[H]
           \small
           \caption{Reverse Diffusion with Our Method}
           \label{alg:cfg++}
            \begin{algorithmic}[1]
             \Require $\x_T \sim \Nc(0, \rmI_d),\textcolor{cfgpp}{\omega(t)} \in \text{C}[0,+\infty)$
             \For{$i=T$ {\bfseries to} $1$}
                 
                 \State{$\epscfgnc(\x_t) = \epsnullnc(\x_t) + \textcolor{cfgpp}{\omega(t)} [\epscond(\x_t) - \epsnullnc(\x_t)]$}
                 \State{$\tweediecfgnc(\x_t) \gets (\x_t - \sqrt{1-\bar\alpha_t}\epscfgnc(\x_t))/\sqrt{\bar\alpha_t}$}
                 \State{$\x_{t-1} = \sqrt{\bar\alpha_{t-1}} \tweediecfgnc(\x_t) + \sqrt{1-\bar\alpha_{t-1}}\epscfgnc(\x_t)$}
            \EndFor
            \State {\bfseries return} $\x_0$
            \end{algorithmic}
    \end{algorithm}
\end{minipage}
\caption{Comparison between reverse diffusion process by CFG and C$^2$FG. Our C$^2$FG guidance weight $\omega(t)$ is a time-decay function. }
\label{algorithm}
\end{figure*}

\section{Additional Experiments}\label{more exp}
\noindent \textbf{More Visualized Analysis on Theorem \ref{thm mse vp}.} In Figure \ref{fig:score ratio}, each pixel in the heatmap corresponds to the logarithmic ratio of the conditional prediction to the unconditional prediction at a specific spatial location and channel. A value of zero (shown as white) indicates no difference (ratio=1). Positive values (\textcolor{red}{red}) indicate amplification of the conditional prediction relative to the unconditional one, while negative values (\textcolor{blue}{blue}) indicate suppression. Importantly, the further a pixel’s value deviates from zero—whether red or blue—the larger the discrepancy between the two predictions. Thus, both strong red and strong blue regions highlight locations where the conditional and unconditional outputs differ most significantly. 

Building on Theorem \ref{thm mse vp}, these heatmaps provide a visual representation of how the score discrepancy evolves over time and across spatial locations. In particular, the early timesteps (larger $t$ indices in the backward diffusion process) show relatively mild color variations, consistent with the theoretical bound $\|\nabla \log p - \nabla \log \tilde p\| \propto \alpha(t)/\sigma^2(t)$, which predicts smaller score differences at well-mixed later times. Conversely, at timesteps closer to the end of the reverse diffusion (smaller $t$ indices), the heatmaps exhibit more pronounced red and blue regions, indicating larger deviations between conditional and unconditional predictions. This aligns with the theoretical observation that the MSE between scores can be large near small diffusion times, where initial distribution differences are amplified. Therefore, the heatmaps not only highlight spatially localized discrepancies but also corroborate the temporal trend predicted by Theorem \ref{thm mse vp}, illustrating that both strong positive (red) and negative (blue) regions correspond to locations and timesteps with significant score mismatch.

\begin{figure}[H]
    \centering
    \includegraphics[width=0.9\linewidth]{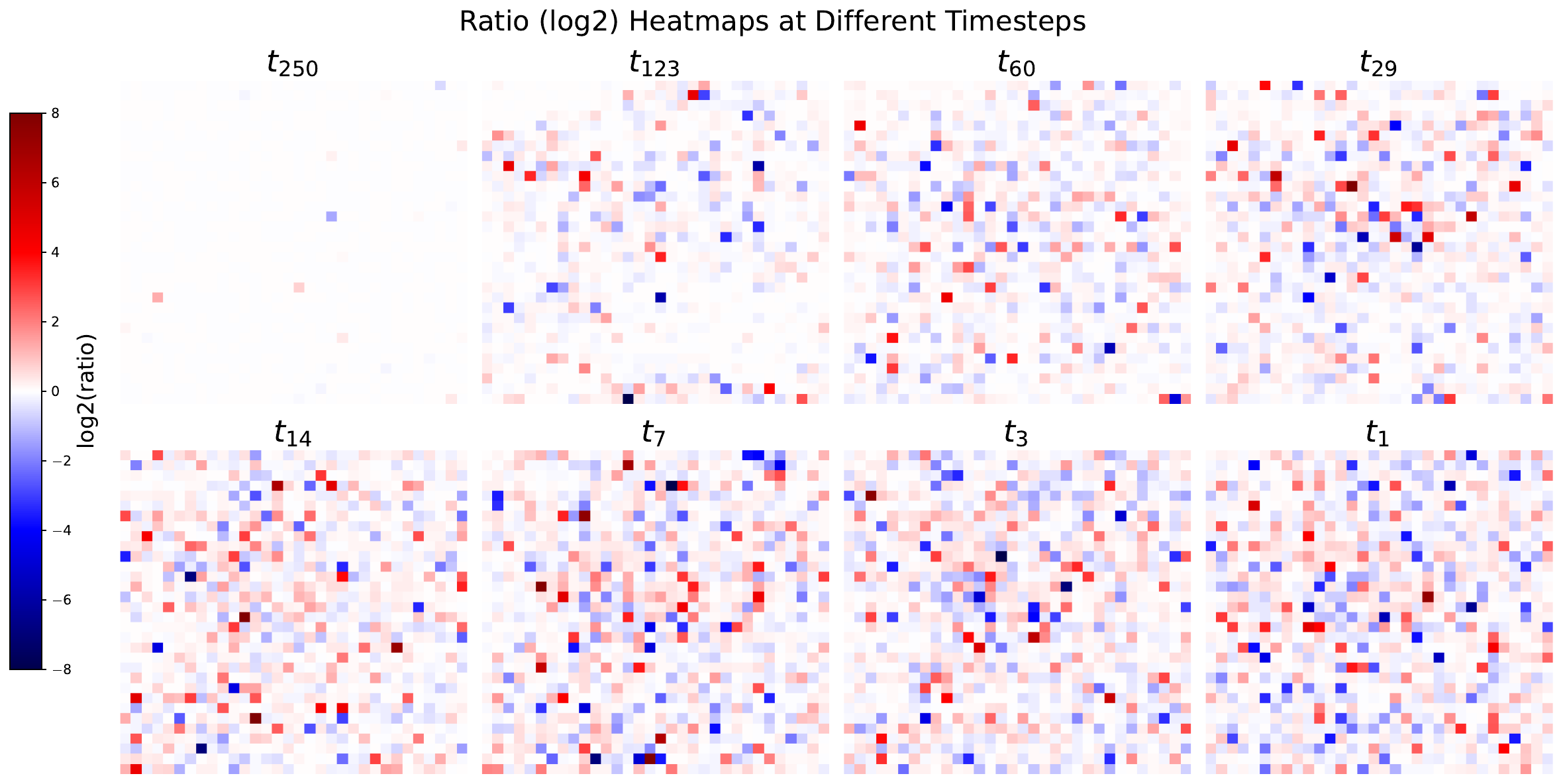}
    \caption{Heatmaps of the logarithmic ratio ($\log_2$) between conditional and unconditional predictions at selected timesteps. White indicates no difference (ratio=1), while \textcolor{red}{red} and \textcolor{blue}{blue} highlight amplification and suppression, respectively. Stronger colors denote larger deviations between the two predictions.}
        \label{fig:score ratio}
\end{figure}
\noindent \textbf{Comparisons of various forms of $\omega(t)$. } As shown in Figure~\ref{fig:IS-fid}, we compare the performance of our method with the DiT-XL/2 baseline~\citep{Peebles2022DiT} under a fixed parameter $\lambda = 1.0$. We observe that the curve corresponding to our method consistently lies below that of the baseline, indicating a strictly better IS–FID trade-off.
In Figure~\ref{fig:IS-fid10k}, we further evaluate several alternative choices of the scheduling function $\omega(t)$ in \cite{wang2024analysis}, including
$\sin\big((t/t_m)\pi \big)$,
$t/t_m$,
$1 - t/t_m$,
together with our proposed formulation($\exp(1 - t/t_m)$), whose trends are shown in Figure \ref{fig:function}.
We observe that certain choices such as sine-based $\omega(t)$ perform even worse than the DiT baseline. Besides, although some of these functions share a broadly similar decreasing trend with our design, they are not aligned with the approximate exponential upper bound derived from our framework. Consequently, their empirical IS–FID trade-off performance is consistently inferior to ours. 
\begin{figure}[!htb]
        \centering
        \includegraphics[width=0.6\linewidth]{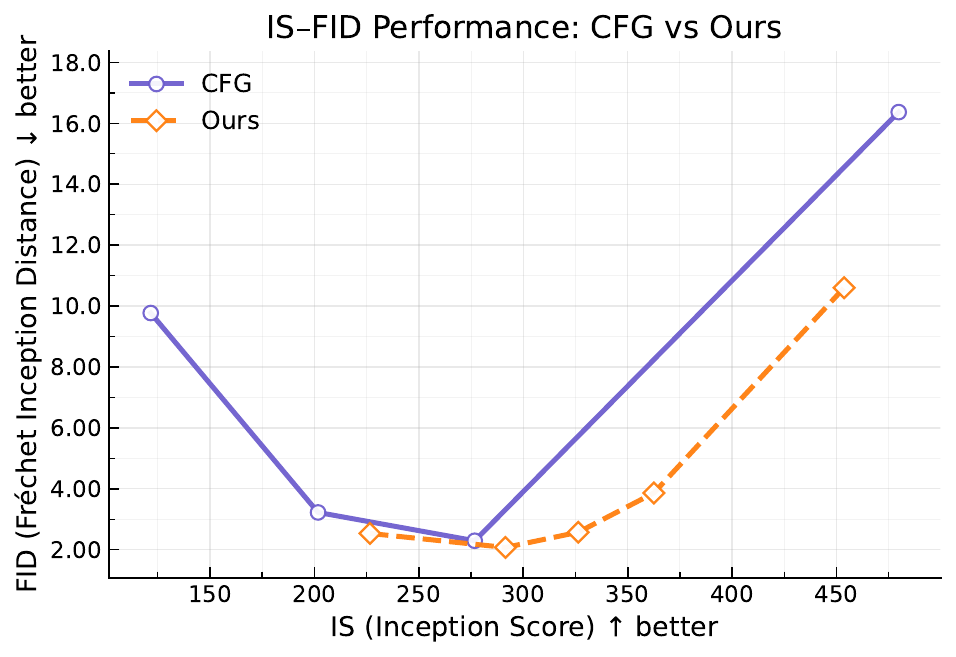}
        \caption{Impact of the initial schedule weight $\omega_0$ on IS–FID performance (with fixed $\lambda=1.0$, 250 inference steps).}
        \label{fig:IS-fid}
\end{figure}

\begin{figure}[!htb]
    \centering
    
    \begin{subfigure}[H]{0.47\textwidth}
        \centering
        \includegraphics[width=\linewidth]{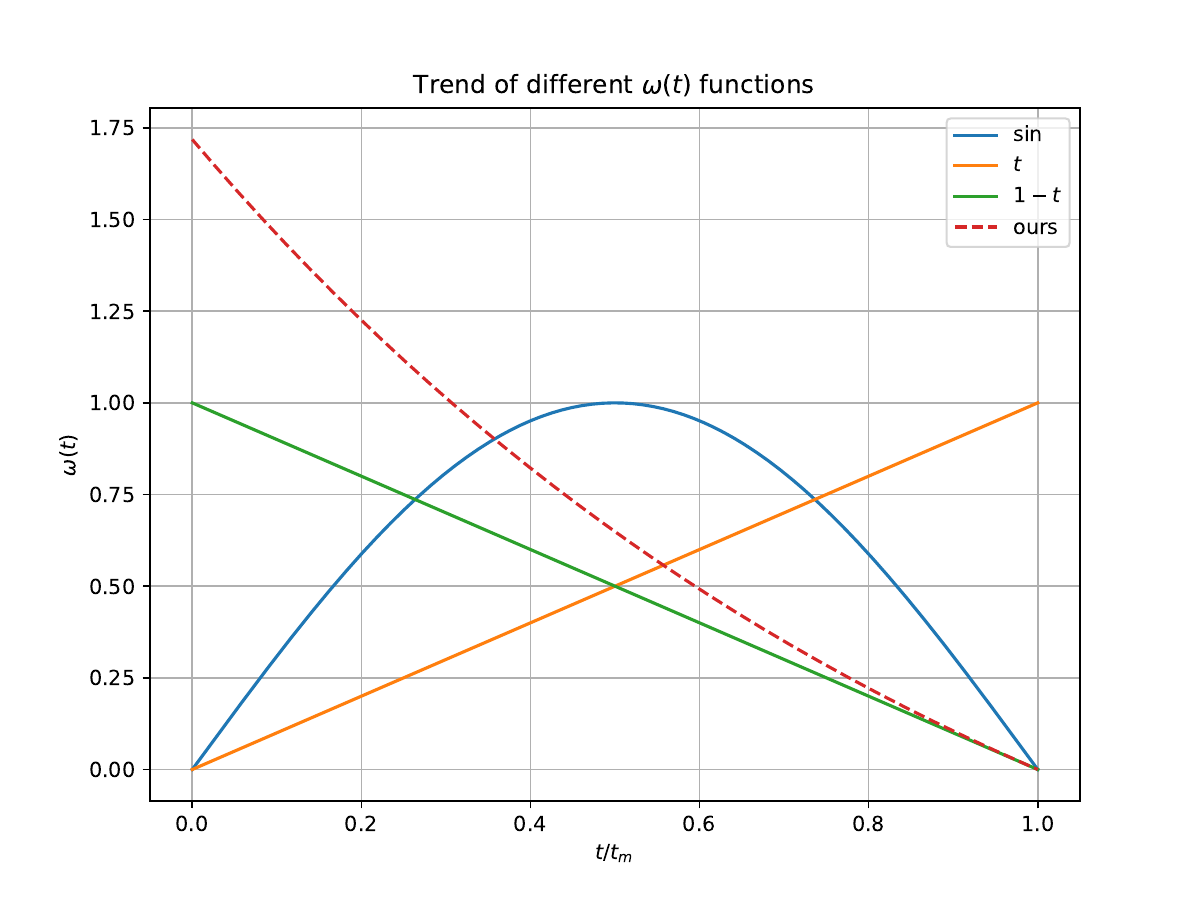}
        \caption{Trend of various $\omega(t)$}
        \label{fig:function}
    \end{subfigure}
    \hfill 
    \begin{subfigure}[H]{0.49\textwidth}
        \centering
        \includegraphics[width=\linewidth]{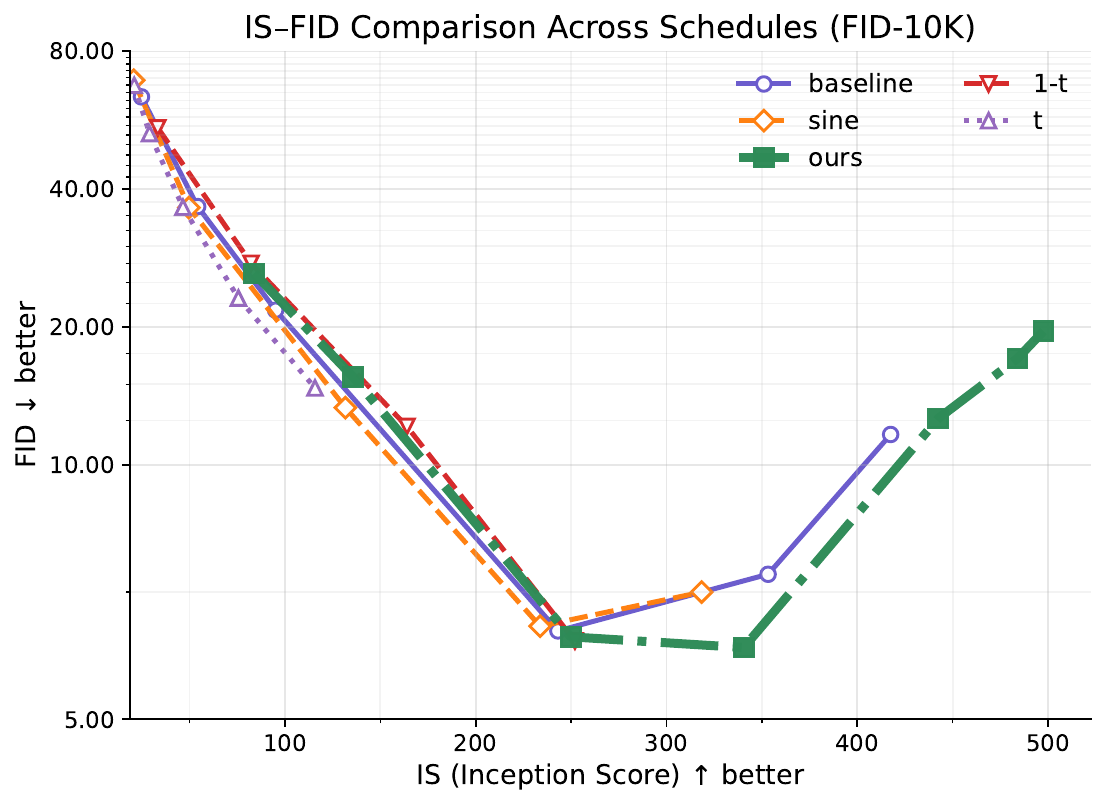}
        \caption{Effect of guidance scale $\omega_0$ on IS–FID curves (50 steps, 10K samples)}
        \label{fig:IS-fid10k}
    \end{subfigure}
    
    \caption{ 
        Comparison of IS–FID performance under different hyperparameter settings on DiT-XL/2 and ImageNet-256. 
    }
\end{figure}

These results highlight that the improvement does not merely come from tuning the scaling magnitude, but primarily from how the temporal modulation interacts with the diffusion dynamics. In particular, our schedule suppresses error amplification in early steps while preserving semantic consistency in later denoising stages, yielding more stable and efficient generation across the entire sampling trajectory.

\begin{table}[!htb]
\vspace{1mm}
\begin{center}
\scalebox{0.9}{
\begin{tabular}{lcccccc}
\toprule
\multicolumn{7}{l}{{\bf\normalsize ImageNet( 256$\times$256 )} } \\
\toprule
\multirow{1}{*}{\bf{Model~~\textcolor{red}{25} inference timesteps}} 
& FID$\downarrow$& IS $\uparrow$ &  Prec$\uparrow$ & Rec$\uparrow$ \\
\midrule
DiT-XL/2 (baseline, $\omega=1.5$)   & 11.88 & 192.13 & 0.7176 &\textbf{0.6651}  \\
DiT-XL/2+$\beta$-CFG ($\omega=1.5, a=b=2.0$)   & 8.42 & 254.42 & 0.8038 &0.6026  \\
DiT-XL/2 + RAAG ($\omega_{\max} =18.0, \alpha=12.0$)     & 21.57 & \textbf{444.57}  & \textbf{0.8360} & 0.2636  \\
\textbf{DiT-XL/2 + Ours} ($\omega_0=1.0, \lambda=1.0$)     & \textbf{7.70}&294.03 & 0.7994 & 0.6357  \\
\bottomrule
\end{tabular}
}
\caption{Performance comparison between our method and existing adaptive CFG approaches on ImageNet-256.}
\label{tab:relatework}
\end{center}
\vspace{-4mm}
\end{table}

To further validate our approach, we compare it against recent time-varying strategies, specifically RAAG~\cite{zhu2025raagratioawareadaptive} and $\beta$-CFG~\cite{malarz2025classifier}. 
We follow the official hyperparameters from the original papers (\(\beta\)-CFG: \(a=b=2.0\); RAAG: \(\omega_{\max}=18.0, \alpha=12.0\)). 
These comparisons are conducted on the DiT-XL/2 model using ImageNet-256, with the results summarized in Table~\ref{tab:relatework}.
The quantitative comparison reveals that, under identical settings, our approach achieves superior overall performance, particularly in term of FID (7.70). While $\beta$-CFG yields marginally higher Precision, it significantly lags in all other metrics. Furthermore, we observe that though RAAG obtains the highest IS score (444.57) and the highest Precision, it suffers from severely degraded FID (21.57) and  Recall (0.2636). In other word, its guidance mechanism emphasizes semantic alignment (IS) rather than accurate distribution fitting (FID), leading to degraded performance. We attribute this to its design focus on text-to-image generation, which appears to generalize poorly to class-conditional settings.  In contrast, our method demonstrates superior generalization capabilities, proving robust across diverse tasks and model architectures.


\noindent \textbf{Analysis of Parameters in C$^2$FG.}
 As shown in Figure \ref{fig:w0} and \ref{fig:lambda}, $\omega_0$ sets the initial or maximum guidance strength, and $\lambda$ controls the rate of exponential decay. 
Moreover, Table \ref{tab:ablation lambda} presents an ablation study on the hyperparameter $\lambda$. While the results demonstrate that various $\lambda$ values are effective for enhancing performance, the best outcome is achieved with $\lambda=\log e=1.0$. The results indicate that this C$^2$FG design is effective. 

\begin{figure}[!htb]
    \centering
    
    \begin{subfigure}[H]{0.48\textwidth}
        \centering
        \includegraphics[width=\linewidth]{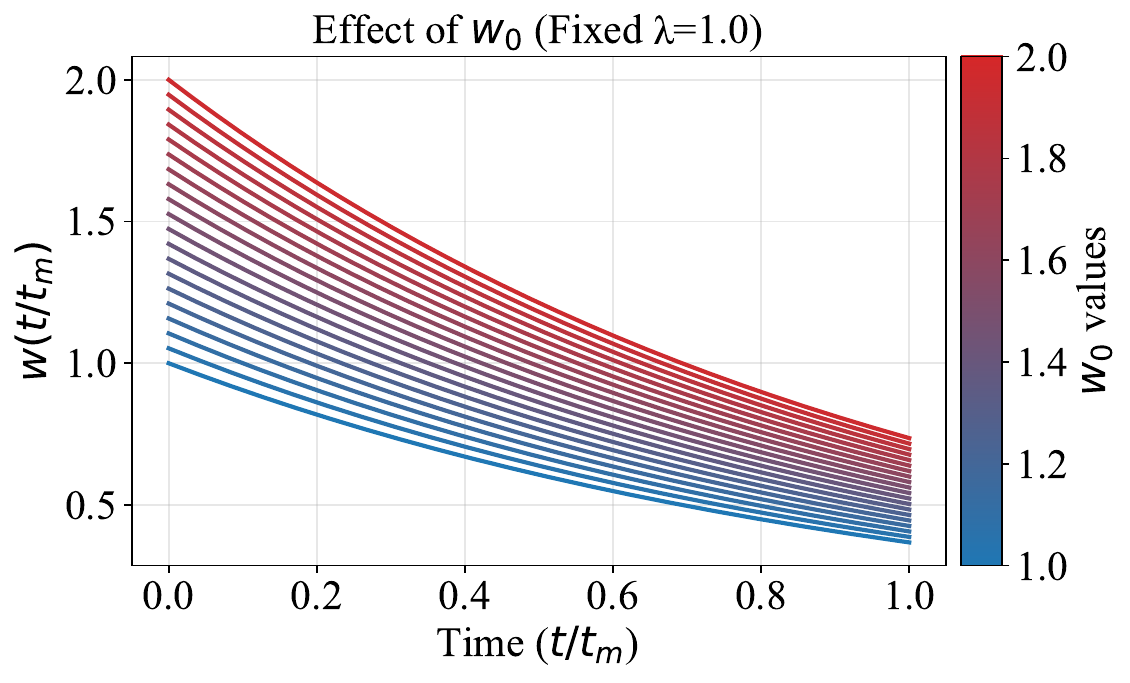}
        \caption{Effect of $\omega_0$ (Fixed $\lambda=1.0$)}
        \label{fig:w0}
    \end{subfigure}
    \hfill 
    \begin{subfigure}[H]{0.48\textwidth}
        \centering
        \includegraphics[width=\linewidth]{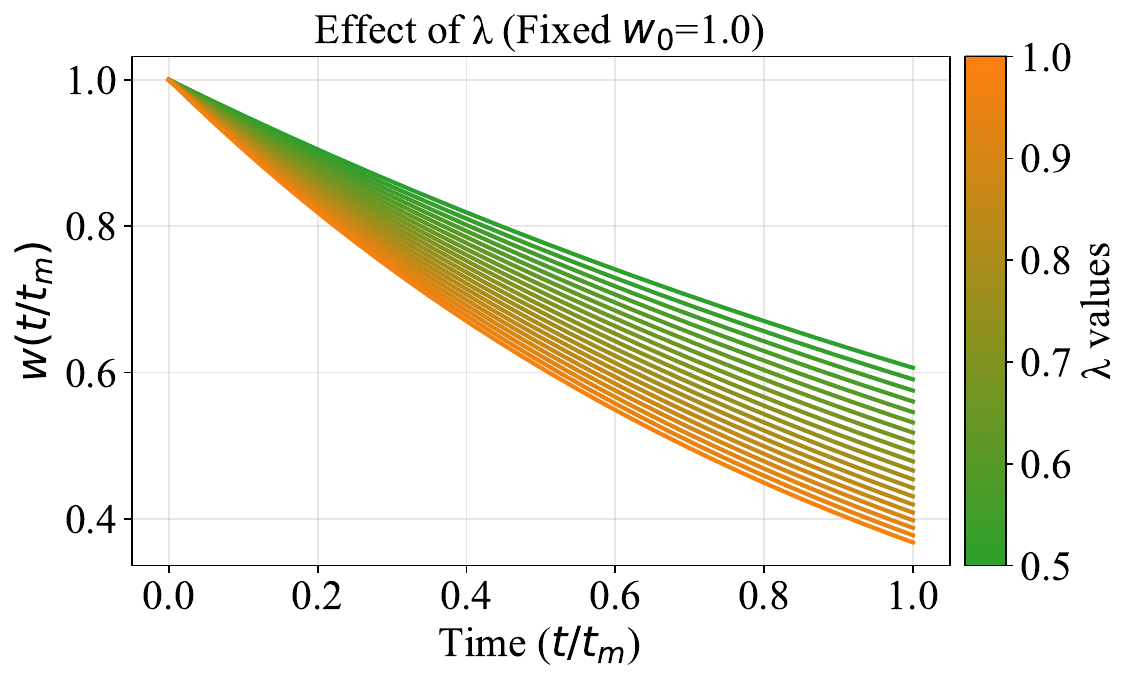}
        \caption{Effect of $\lambda$ (Fixed $\omega_0=1.0$)}
        \label{fig:lambda}
    \end{subfigure}
    
    \caption{ 
        (a) demonstrates the impact of initial weight $\omega_0$; 
        (b) illustrates how different $\lambda$ values affect the decay profile.
    }
\end{figure}

\begin{table}[ht]
\centering
\begin{tabular}{lc}
\toprule
\multicolumn{2}{l}{\bf ImageNet(256$\times$256), 50k samples, 250 SDE inference timesteps} \\
\midrule
Model & FID$\downarrow$ \\
\midrule
REPA (Fixed CFG = 1.35) & 1.80 \\
REPA ($\lambda=\log2$) & 1.68 \\
\textbf{REPA ($\lambda=1 (\log e)$)} & \textbf{1.51} \\
REPA ($\lambda=\log3$) & 1.58 \\
\bottomrule
\end{tabular}

\captionsetup{width=0.9\linewidth} 
\caption{Comparison between the different effect of $\lambda$, fixing $\omega_0=1.0$.}
\label{tab:ablation lambda}
\end{table}


\noindent \textbf{Results on More Framework.} In Table \ref{tab:ablation edm2}, we show the results of our C$^2$FG on autoguidance introduced by \cite{Karras2024autoguidance} with the model of EDM2 \citep{Karras2024edm2}. Autoguidance  involves two denoiser networks $D_0(x;\sigma,c)$ and $D_1(x;\sigma,c)$ and the  guiding effect is achieved by extrapolating between the two denoising results by a factor $\omega$:
$$D_{\omega}(x;\sigma_t,c) = \omega D_1(x;\sigma_t,c) + (1-\omega) D_0(x;\sigma_t,c),$$
based on their method, we make $\omega$ be a time-variance function $\omega(t)$ with the same formula of C$^2$FG: $\omega(t)=\omega_0\exp(1-t/t_{\max})$. As shown in Table \ref{tab:ablation edm2}, our dynamic guidance $\omega(t)$ consistently improves over the static guidance baseline. On ImageNet-64, where the model operates directly in the pixel domain, our method achieves lower FID and FD-DINOv2 \citep{Justin2024fddino}, indicating that dynamic weighting not only preserves fidelity but also enhances semantic alignment. On high-resolution ImageNet-512, which is considerably more challenging, we also observe clear gains under the same setting, confirming that the proposed C$^2$FG can robustly integrate with autoguidance across scales. These results highlight the generality of our approach: the time-dependent extrapolation scheme provides a more adaptive balance between fidelity and diversity than a fixed scalar weight.
\begin{table}[ht]
\centering
\begin{tabular}{lcc}
\toprule
\multicolumn{2}{l}{\bf ImageNet(64$\times$64)} \\
\midrule
Model & FID$\downarrow$ & $\textnormal{FD}_\textnormal{DINOv2}\downarrow$\\
\midrule
EDM2-S-autoguidance ($\omega=1.7$) & 1.044 & 56.3 \\
\textbf{EDM2-S-autoguidance+Ours}($\omega_0=0.9,\lambda=0.7$) & \textbf{1.028} & \textbf{52.7} \\
\midrule
\multirow{1}{*}{\bf ImageNet (512$\times$512), 10k samples} \\
\midrule
EDM2-S-autoguidance ($\omega=1.4$) & 5.27 & 121.2\\
\textbf{EDM2-S-autoguidance+Ours}($\omega_0=0.9,\lambda=0.5$) & \textbf{5.15} & \textbf{116.7} \\
\bottomrule
\end{tabular}
\captionsetup{width=0.9\linewidth} 
\caption{We evaluated conditional image generation on ImageNet with EDM2 and Autoguidance.}
\label{tab:ablation edm2}
\end{table}

\noindent \textbf{Denoising Process.}
As shown in Figure \ref{fig:compare}, we provide a qualitative comparison of intermediate decoding results between our C$^2$FG and the baseline across the denoising trajectory. From step 250 down to 50, both methods generate visually similar results. However, in the final refinement stage (from step 50 to 0), the difference becomes more pronounced: our C$^2$FG produces sharper structures and more coherent details, highlighting the benefit of dynamically adjusting the guidance strength in the later denoising steps.
\begin{figure}[!htb]
    \centering
    \includegraphics[width=0.85\linewidth]{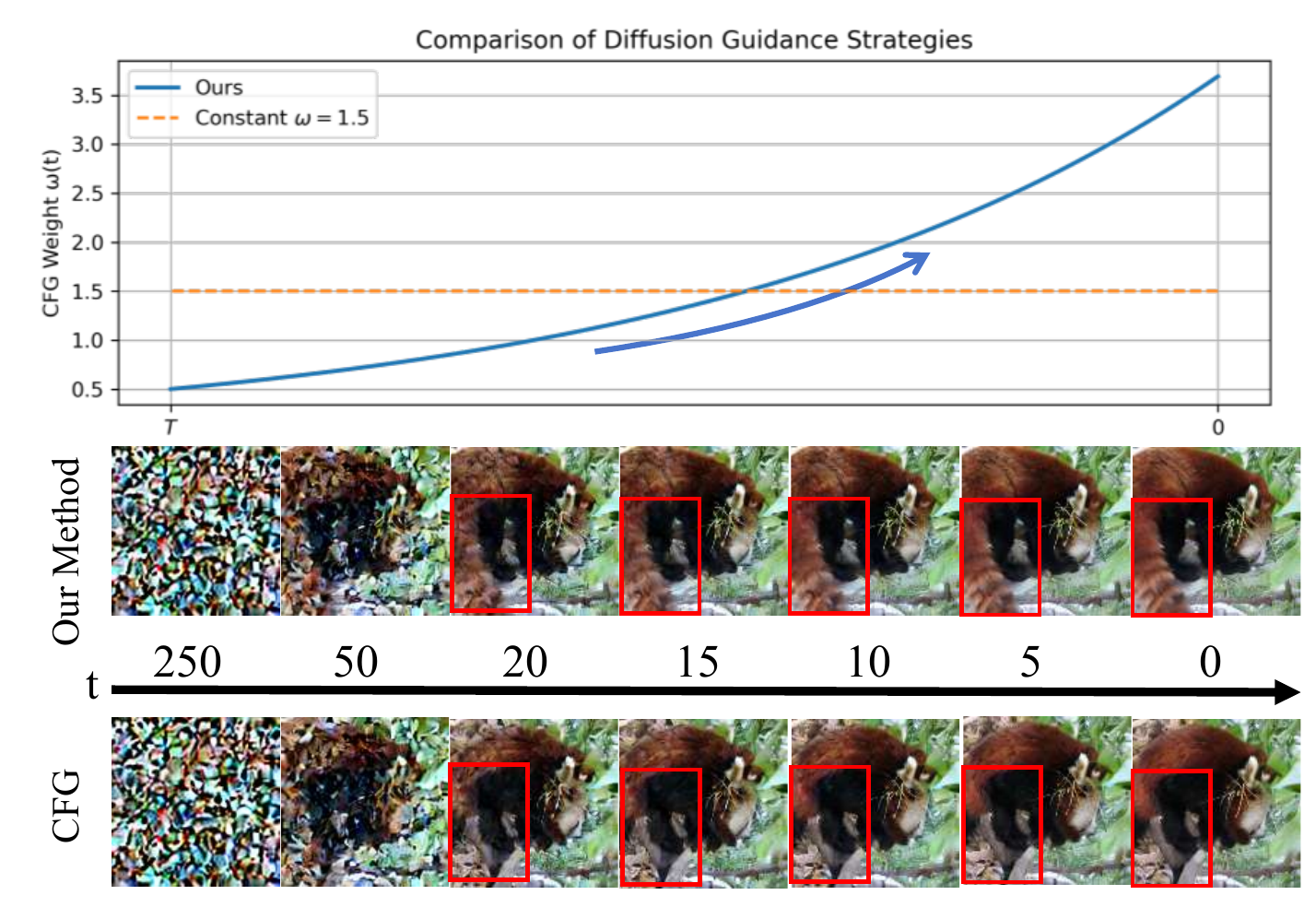}
    \caption{Comparison between results during the denoising process  of C$^2$FG and Baseline.}
    \label{fig:compare}
\end{figure}

\noindent \textbf{Additional Results.}
  In Table \ref{tab:table} we compare our methods with other methods on different models. On SD1.5 \cite{rombach2022high}, C$^2$FG achieves the \textbf{best FID\&CLIP}, surpassing CFG++ \cite{chung2025cfg} and $\beta$-CFG \cite{malarz2025classifier}. It also consistently improves Flux \cite{labs2025flux1kontextflowmatching} and SD3 \cite{esser2024scalingrectifiedflowtransformers}. On SiT \cite{yu2025repa}, C$^2$FG outperforms FDG \cite{Sadat2025GuidanceIT}. Thus C$^2$FG consistently outperforms other methods across diverse tasks, verifying its strong generality. And Figure \ref{fig:img} shows our visualized results on T2I tasks. Additionally, Figure \ref{fig:vis} shows additional results using our C$^2$FG method on DiT and SiT models.

\begin{table}[t]
\centering
\caption{Additional Comparisons. \textbf{Left:} Comparisons with dynamic guidance methods on SD1.5 (MS-COCO) and SiT  (ImageNet). \textbf{Right:} Results on modern T2I models (Flux, SD3).}
\label{tab:table}
\resizebox{0.9\textwidth}{!}{
\begin{tabular}{c c c c c |c c  | c c c}
\toprule
Compare&\multicolumn{4}{c}{Fixed SD1.5,MS-COCO} &\multicolumn{2}{c}{Fixed SiT,ImageNet} &\multicolumn{3}{c}{Compare T2I models (CLIP$\uparrow$)} \\
\cmidrule(lr){2-5} \cmidrule(lr){6-7} \cmidrule(lr){8-10}
Method & CFG & CFG++ & $\beta$-CFG & C$^2$FG $\textbf{(4,1)}$ &FDG&C$^2$FG\textbf{(1.7,0.15)}& Models & Flux $\textbf{(1.5,1)}$ & SD3 $\textbf{(5,1)}$ \\

\midrule
FID(10k) $\downarrow$
& 19.32 & 18.87 & 16.74 & \textbf{16.71} &6.15&\textbf{3.20}& CFG & 31.4 & 31.4 \\
CLIP(10k) $\uparrow$
& 32.0 & 32.0 & 31.7 & \textbf{32.0} &--&--& C$^2$FG
& \textbf{31.5} & \textbf{31.5} \\
\bottomrule
\end{tabular}

}

\end{table}

\begin{figure}[tb] \centering
    \caption{\textbf{Visual Comparison.} C$^2$FG produces images that better align with the text prompt  than standard CFG, yielding more faithful details, consistent with the quantitative gains in Table \ref{tab:table}.} \label{fig:img}
    \includegraphics[width=0.8\textwidth]{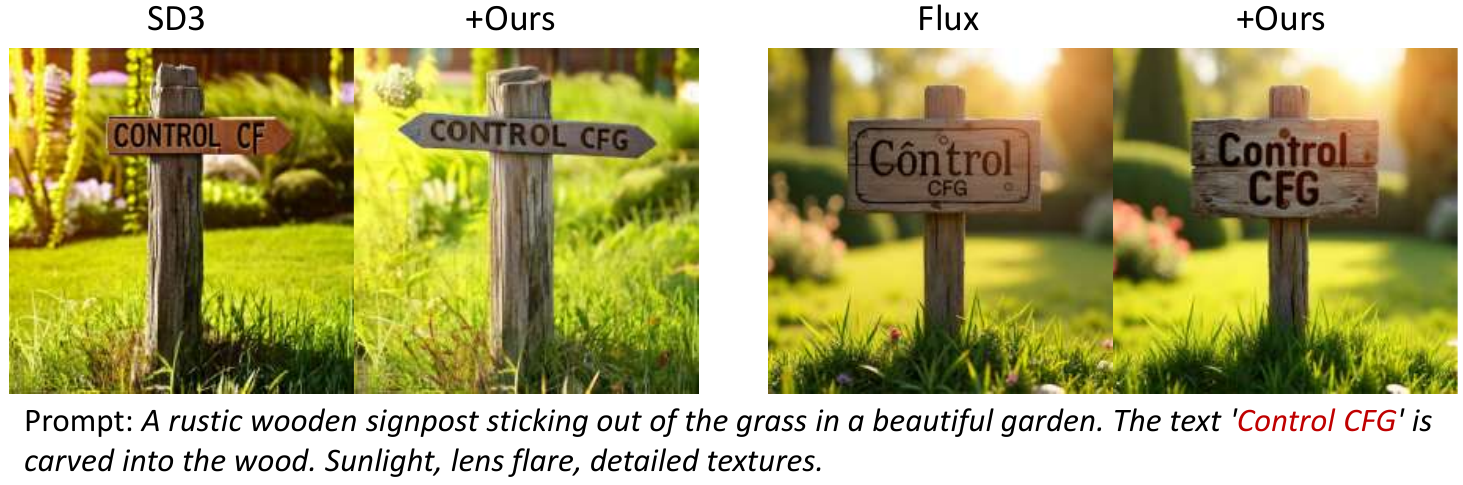}
    
\end{figure}

\begin{figure}[H]
    \centering
    \begin{subfigure}[H]{0.9\textwidth}
        \centering
        \includegraphics[width=\linewidth]{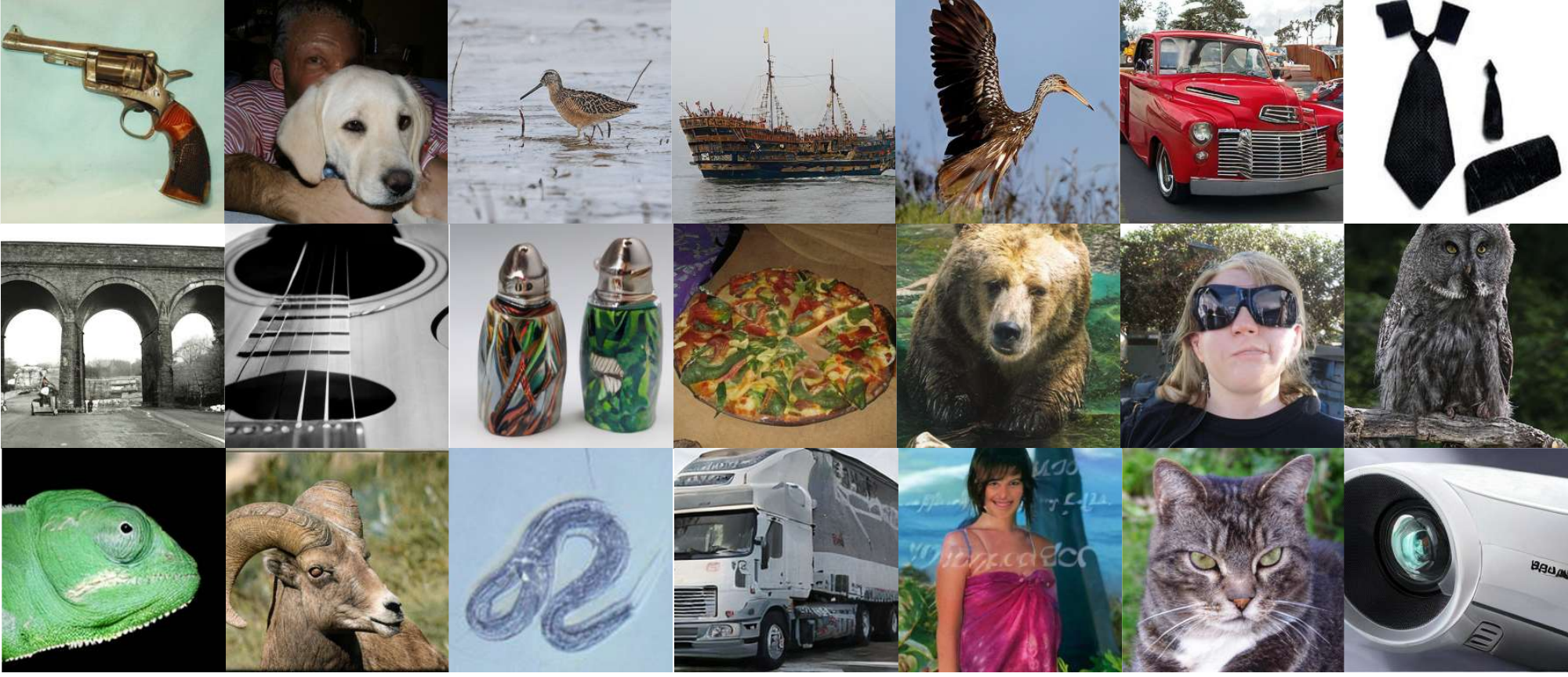}
        \caption{Images generated by the DiT-XL/2 model with C$^2$FG on ImageNet-256.}
        \label{fig:vis_dit}
    \end{subfigure}
    
    
    \centering
    \begin{subfigure}[H]{0.9\textwidth}
        \centering
        \includegraphics[width=\linewidth]{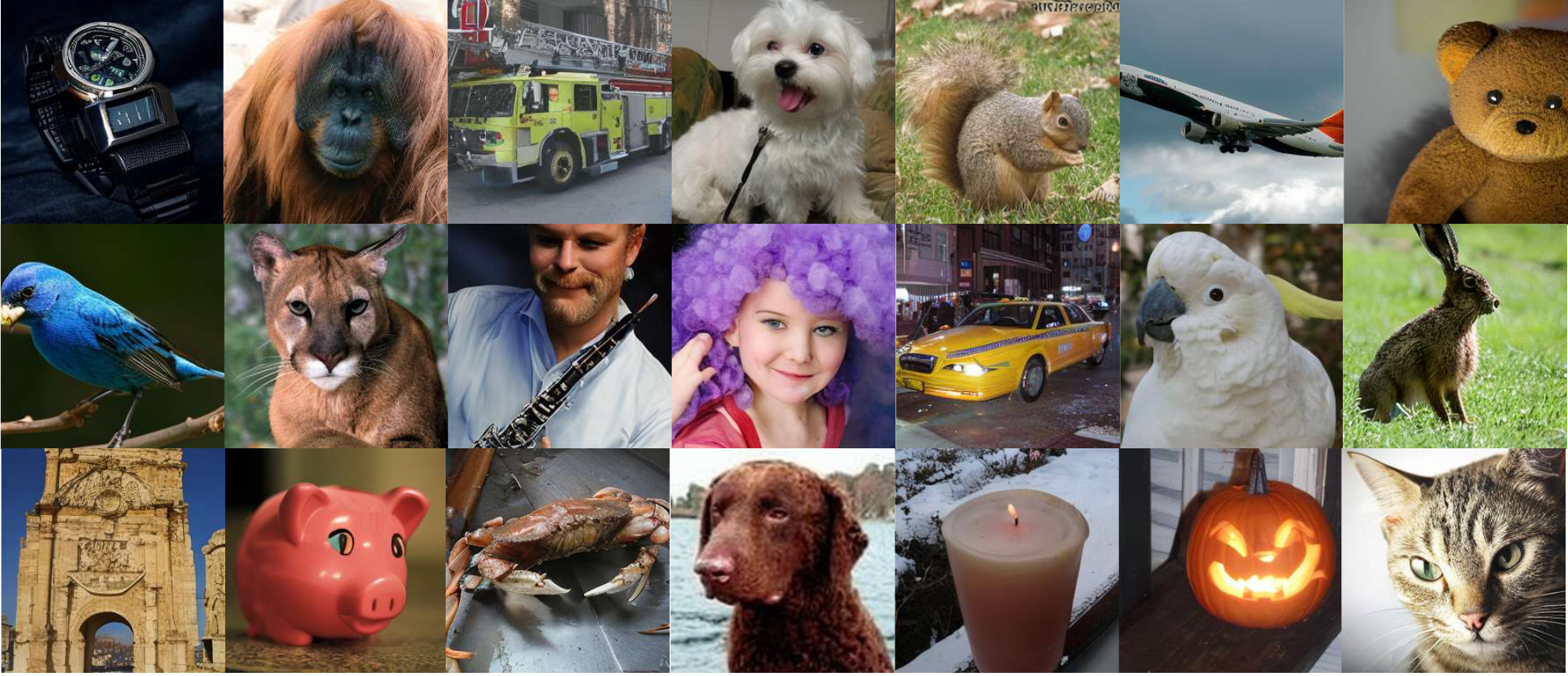}
        \caption{Images generated by the SiT-XL/2 (REPA) model with C$^2$FG on ImageNet-256.}
        \label{fig:vis_repa}
    \end{subfigure}
    
    \caption{ 
         Additional results for C$^2$FG.
    }
    \label{fig:vis}
\end{figure}


\end{document}